\documentclass[letterpaper]{article} 
\usepackage{aaai23}  
\usepackage{times}  
\usepackage{helvet}  
\usepackage{courier}  
\usepackage[hyphens]{url}  
\usepackage{graphicx} 
\usepackage{natbib}  
\usepackage{caption} 
\usepackage{etoolbox}
\newtoggle{arxiv}  
\toggletrue{arxiv}
\usepackage{algorithm}
\usepackage{algorithmic}

\usepackage{newfloat}
\usepackage{listings}
\DeclareCaptionStyle{ruled}{labelfont=normalfont,labelsep=colon,strut=off} 
\floatstyle{ruled}
\newfloat{listing}{tb}{lst}{}
\floatname{listing}{Listing}
\usepackage{amsmath}
\usepackage[export]{adjustbox}

\DeclareMathOperator*{\argmin}{arg\,min}

\newcommand{\ycomment}[2]{({\color{red}YT: #2})}

\newcommand{\ttl}{Resource Sharing Through Multi-Round Matchings}
\newcommand{\reward}{\mu}
\newcommand{\increward}{\delta}
\newcommand{\totreward}{\mathbb{B}}

\newcommand{\ttbig}{\mathtt{big}}

\usepackage{tikz}

\usepackage{enumitem}
\usepackage{amsmath}
\usepackage{amsfonts}
\usepackage{amssymb}
\usepackage{amsthm}

\newtheorem{theorem}{Theorem}[section]

\newtheorem{lemma}[theorem]{Lemma}{\bfseries}{\itshape}
\newtheorem{proposition}[theorem]
        {Proposition}{\bfseries}{\itshape}
\newtheorem{corollary}[theorem]
        {Corollary}{\bfseries}{\itshape}
\newtheorem{definition}[theorem]{Definition}

\newcommand{\maxtbmrm}{\mbox{\textsc{MaxTB-MRM}}}
\newcommand{\maxsamrm}{\mbox{\textsc{MaxSA-MRM}}}
\newcommand{\agmaxsamrm}{\mbox{\textsc{AG-MaxSA-MRM}}}
\newcommand{\ilpagmaxsamrm}
     {\mbox{\textsc{ILP-AG-MaxSA-MRM}}}
\newcommand{\heuristicagmrm}{\textsc{PS-AG-MaxSA-MRM}}

\newcommand{\agmrm}{\textsc{AG-MRM}}
\newcommand{\mrm}{\textsc{MRM}}

\newcommand{\msmrm}{\textsc{MaxSA-MRM}}

\newcommand{\rmmrm}{\textsc{RM-MRM}}
\newcommand{\tpmrm}{\textsc{TP-MRM}}
\newcommand{\msagmrm}{\textsc{MS-AG-MRM}}

\newcommand{\ilpmsagmrm}{\textsc{ILP-Max-AG-MRM}}

\newcommand{\agset}{X} 
 
\newcommand{\prin}{\mbox{\textsc{Principal}}}

\newcommand{\mbsoln}{\mathcal{M}}

\newcommand{\rset}{Y} 

\newcommand{\mvcc}{\mbox{\textsc{MVC-Cubic}}}
\newcommand{\algmrm}{\mbox{\textsc{Alg-MRM}}}
\newcommand{\algmaxtbmrm}{\mbox{\textsc{Alg-MaxTB-MRM}}}

\newcommand{\calc}{\mbox{$\mathcal{C}$}}

\newcommand{\compgr}{\mbox{$G(\agset,\,\rset,\,E)$}}

\newcommand{\bbx}{\mbox{$\agset$}}
\newcommand{\bby}{\mbox{$\rset$}}

\newcommand{\cnp}{\textbf{NP}}

\title{\ttl}
\author{Yohai Trabelsi, \textsuperscript{\rm 1}
Abhijin Adiga,\textsuperscript{\rm 2}
Sarit Kraus,\textsuperscript{\rm 1}
S.~S. Ravi,\textsuperscript{\rm 2, \rm 3}
Daniel J. Rosenkrantz\textsuperscript{\rm 2, \rm 3}
}
\affiliations {
    \textsuperscript{\rm 1} Department of Computer Science, 
Bar-Ilan University, Ramat Gan, Israel\\
    \textsuperscript{\rm 2} Biocomplexity Institute and Initiative,
Univ. of Virginia, Charlottesville, VA, USA\\
    \textsuperscript{\rm 3} Dept. of Computer Science, University at Albany -- SUNY, Albany, NY, USA\\
    yohai.trabelsi@gmail.com, 
    abhijin@virginia.edu, 
    sarit@cs.biu.ac.il, 
    ssravi0@gmail.com, 
    drosenkrantz@gmail.com
}

\newcommand{\cc}{\mbox{\emph{Course-Classroom}}}
\newcommand{\ls}{\mbox{\emph{Lab-Space}}}

\iftoggle{arxiv}
{
\nocopyright
}
{}

\begin{document}

\maketitle 

\begin{abstract}
Applications such as employees sharing office spaces over a workweek
can be modeled as problems where agents are matched to resources
over multiple rounds. Agents' requirements limit the set of compatible
resources and the rounds in which they want to be matched.  Viewing such an
application as a multi-round matching problem on a bipartite compatibility
graph between agents and resources, we show that a  solution 
(i.e., a set of matchings, with one matching per round) can be found
efficiently if one exists.  To cope with situations where a solution does not exist, we consider two extensions. In
the first extension, a benefit function is defined for each agent and the
objective is to find a multi-round matching to maximize the total benefit.  For a
general class of benefit functions satisfying certain properties (including
diminishing returns), we show that this multi-round matching problem is
efficiently solvable.  This class includes utilitarian and Rawlsian welfare
functions.  
For another benefit function, we show that the maximization
problem is NP-hard.  
In the second extension, the objective is to generate advice to
each agent (i.e., a subset of requirements to be relaxed) subject to a
budget constraint so that the agent can be matched.
We show that this budget-constrained advice generation problem is NP-hard.
For this problem, we develop an integer linear programming formulation  as well
as a heuristic based on local search.
We experimentally evaluate our algorithms on
synthetic networks and apply them to two real-world situations: shared
office spaces and matching courses to classrooms.
\end{abstract}

\section{Introduction}
\label{sec:intro}


We consider resource allocation problems that arise in practical
applications such as 
hot desking or shared work spaces~\cite{varone2019dataset,cai2010common},
classroom scheduling~\cite{phillips2015integer},
matching customers with
taxicabs~\cite{karamanis2020assignment,kucharski2020exact}, 
and matching agricultural equipment 
with farms~\cite{Gilbert-2018,RS-2020}.
In such
scenarios, many agents (individuals, cohorts, farms, or in general, entities) are
competing for a limited number of time-shared resources.  
In our formulation, an agent can be matched to at most
one resource in any time slot, but might want to be
matched in more than one time slot.  Agents may have some \emph{restrictions}
that limit the set of resources to which they can be matched or
possible time slots in which they can be matched.  Any resource whose
specifications do not meet an agent's restrictions is \emph{incompatible} with that agent.

In light of the COVID-19 pandemic, such resource allocation problems have
become as important as ever. For example, social-distancing requirements
(such as maintaining a six-foot separation between individuals) led to
a dramatic decrease in the number of individuals who can occupy an enclosed
space, whether it is a classroom~\cite{dmg2020,enriching2020},
workspace~\cite{parker2020covid} or visitation room~\cite{wong2021self}.

We model this resource allocation problem as a $k$-round matching problem
on a bipartite graph, where the two node sets represent the agents and
resources respectively.  
We assume that the rounds are numbered 1 through $k$,
and each round matches a subset of agents with a subset of
resources.
Each agent specifies the set of permissible 
rounds in which it can participate and the
desired number of rounds (or matchings) in which it needs to be assigned a resource. 
Consider for example a classroom scheduling for a
workweek ($k=5$). Each lecturer who wants to schedule class sessions for her courses 
specifies the number of sessions she would like to schedule and the possible weekdays. There may also be additional requirements for classrooms (e.g., room size, location, 
computer lab).
As a result, some
agent-resource pairs become incompatible. 
The \prin{}'s objective is to
find a set of at most $k$ matchings satisfying all
the requirements, if one exists.
In the classroom scheduling example, the existence of such a set means that there exists a solution where each lecturer receives the desired number of sessions on the desired days.

We refer to this as the \textbf{multi-round matching problem} (\mrm). 
We consider two additional problem formulations to cope with the situation when a solution satisfying all the requirements does not exist. 

In the first formulation, 
an agent receives a benefit (or reward) that depends
on the number of rounds where it is matched, and
the objective is to find a solution that maximizes 
the sum of the benefits
over all agents.  
We refer to this as the \textbf{multi-round matching 
for total benefit maximization} problem
(abbreviated as \maxtbmrm{}). 
When applying this formulation to the classroom scheduling example, it is possible to find a solution in which the number of assigned sessions is maximized (by maximizing the corresponding utility function). Alternatively, we can also find a solution in which the minimal assignment ratio for a lecturer (i.e., the number of assigned sessions divided by the number requested) is maximized.

In the second formulation, the objective is to generate suitable advice to
each agent to relax its resource requirements so that it 
becomes compatible with previously incompatible resources. 
However, the agent incurs a cost to
relax its restrictions (e.g., social distancing requirements not met,
increase in travel time).
So, the suggested advice must satisfy the budget
constraints that are chosen by the agents.
In other words, the goal of this problem is to suggest
relaxations to agents' requirements, subject to budget constraints, so that in the new bipartite graph
(obtained after relaxing the requirements), all the
agents' requirements are satisfied.
In scheduling classrooms, for example, the~\prin{} may require that some lecturers waive a subset of their restrictions (e.g., having a far-away room) so that they can get the number of sessions requested.
We refer to this 
to as the \textbf{advice generation for multi-round matching} problem (abbreviated as \agmrm{}).

\smallskip

\noindent
\textbf{Summary of contributions:}

\smallskip

\noindent
\underline{(a) An efficient algorithm for \maxtbmrm{}.}
For a general class of benefit functions that satisfy 
certain properties including monotonicity (i.e., the function is
non-decreasing with increase in the number of rounds)
and diminishing returns (i.e., increase in the value of the
function becomes smaller as the number of rounds is increased),
we show that the \maxtbmrm{} problem can be
solved efficiently by a reduction to the maximum weighted matching problem.
A simple example of such a function (where each agent receives a benefit of 1 for each round in which it is
matched) represents a \textbf{utilitarian social welfare}
function \cite{viner1949bentham}.
Our efficient algorithm for this problem yields
as a corollary an efficient algorithm for the \mrm{}
problem mentioned above.
Our algorithm can also be used for a more complex benefit function that models a \textbf{Rawlsian social welfare function} 
\cite{rawls1999theory,Stark-2020},
where the goal is to maximize the minimum satisfaction
ratio (i.e., the ratio of the number of rounds assigned to an agent to the requested number of rounds) over all
the agents.

\smallskip
\noindent
\underline{(b) Maximizing the number of satisfied agents.}
Given a multi-round matching, we say that an agent is \textbf{satisfied} if the matching satisfies all
the requirements of the agent.
The objective of finding a multi-round matching that satisfies the largest number of agents can be modeled as the problem of maximizing
the total benefit by specifying a simple benefit
function for each agent. However, such a benefit function
\emph{doesn't} have the diminishing returns property.
We show that this optimization problem 
(denoted by \maxsamrm{}) is \cnp-hard,
even when each agent needs to be matched in at most three
rounds. 
This result points out the importance of the
diminishing returns property in obtaining an efficient
algorithm for maximizing the total benefit.

\noindent
\underline{(c) Advice generation.} 
We show that \agmrm{}  is \cnp-hard even when
there is only one agent who needs to be matched 
in two or more rounds. 
Recall that the \agmrm{} problem requires
that each agent must be satisfied (i.e., assigned the desired
number of rounds) in the new
compatibility graph (obtained by relaxing the suggested
restrictions).
It is interesting to note that the hardness of \agmrm{} 
is due to the advice generation part and not the computation of matching. 
(Without the advice generation part, the \agmrm{} problem 
corresponds to \mrm{} on a given compatibility graph, 
which is efficiently solvable as mentioned above.)
The hardness of \agmrm{} directly implies
the hardness of the problem (denoted by \agmaxsamrm{})
where the generated advice must lead to a matching
that satisfies the maximum number of agents.
We present two solution
approaches for the \agmaxsamrm{} problem: 
(i)~an integer linear
program (\ilpagmaxsamrm{})
to find an optimal solution and 
(ii)~a pruned local search heuristic (\heuristicagmrm{}) that uses our algorithm for \mrm{} to
generate solutions of good quality.

\smallskip

\noindent
\underline{(d) Experimental results.}
We present a back-to-the-lab desk-sharing study that has been conducted in
the AI lab at a university to facilitate lab personnel 
intending to return to
the work place during the COVID-19 epidemic.
This study applies our algorithms to guide policies
for returning to work.  
In addition, we present an experimental evaluation
of our algorithms on several synthetic data sets as well as
on a data set for matching courses to classrooms. 

Table~\ref{tab:prob_considered} shows the list of problems 
considered in our work and our main results.
\iftoggle{arxiv}
{
}
{
Due to space limitations, many proofs are omitted;
they can be founded  in an expanded version~\cite{mround-arxiv}.
}

\begin{table}
{\small
\begin{tabular}{|p{0.35\columnwidth}|p{0.55\columnwidth}|}\hline
\textbf{Problem} & \textbf{Results} \\ \hline\hline

\maxtbmrm{} & {Efficient algorithm for a general\newline class of benefit functions. An efficient algorithm for the~\mrm{} 
problem is a corollary.} \\ \hline
\maxsamrm{} & {\cnp-hard (reduction from
             Minimum Vertex Cover for cubic graphs)} \\ \hline         
\agmrm{}  &  {\cnp-hard (reduction from Minimum Set Cover)}\\ 
\agmaxsamrm{} & {ILP and a local search heuristic for
                 \agmaxsamrm{}}
                \\ \hline
\end{tabular}
}
\caption{Overview of problems and main results}
\label{tab:prob_considered}
\end{table}

\section{Related Work}
\label{sse:related}

Resource allocation in multi-agent systems
is a well-studied area (e.g.,
\cite{Chevaleyre-etal-2006,Gorodetski-etal-2003,Dolgov-etal-2006}).  The
general focus of this work is on topics such as how agents express
their requirements, algorithms for allocating resources and evaluating the quality of the
resulting allocations.  Some complexity and  approximability issues
in this context are discussed in Nguyen~et~al. (\citeyear{nguyen2013survey}).  Zahedi
et al. (\citeyear{Zahedi-etal-2020}) study the allocation of
tasks to agents so that the task allocator can answer
queries on counterfactual allocations.

Babaei et al.~(\citeyear{babaei2015survey}) presented a survey of approaches to solving timetabling problems that assign university courses to spaces and time slots, avoiding the violation of hard constraints and trying to satisfy soft constraints as much as possible. A few papers proposed approaches to exam timetabling (e.g., Leite et al.~(\citeyear{leite2019fast}), Elley~(\citeyear{eley2006ant})).
Other work addressed timetabling for high schools (e.g., Fonseca et al.~(\citeyear{fonseca2016integrating}) and Tan et al.~(\citeyear{tan2021survey})).
However, the constraints allowed in the various timetabling problem variations are much more complex than in our setting.

Variants of multi-round matching have been considered in game
theoretic settings. Anshelevich~et al.~(\citeyear{anshelevich2013social})
consider an online matching algorithm where incoming agents are matched in
batches. 
Liu~(\citeyear{liu2020stability}) considers matching a set of long-lived players (akin
to resources in our setting) to short-lived players 
in time-slots. These two works  focus on mechanism design to achieve stability and maximize social welfare.  
Gollapudi et al. (\citeyear{gollapudi2020almost}) and also Caragiannis and Narang (\citeyear{caragiannis2022repeatedly}) consider repeated matching 
with the objective of achieving envy-freeness. However, the latter only considers the case of perfect matching, where agents are matched exactly once in each round. S{\"u}hr~et al.~(\citeyear{suhr2019two}) considered optimizing fairness in repeated matchings motivated by applications to the ride-hailing problem.

Zanker~et~al. (\citeyear{Zanker-etal-2010}) discuss the design and evaluation of
recommendation systems that allow users to specify
soft constraints regarding products of interest.  
A good discussion on the design of constraint-based recommendation systems  
appears in Felfernig~et~al. (\citeyear{Felfernig-etal-2011}).

Zhou and Han (\citeyear{Zhou-Han-2019}) propose an
approach for a graph-based recommendation system that forms groups of
agents with similar preferences to allocate resources.
Trabelsi et al.~(\citeyear{trabelsi2022maximizing}, \citeyear{trabelsi2022maximizing2}) discussed the problem of advising an agent to modify its preferences so that it can be assigned 
a resource in a maximum matching. In this work, however, the advice is given to a single agent and only one matching is generated.

To our knowledge, the problems studied in our paper, 
namely finding multi-round matchings that optimize general benefit
functions and advising 
agents to modify their preferences so that they can be
assigned resources in such matchings,
have  not been addressed in the literature.

\newcommand{\rxbinary}{\mathbf{b}}
\newcommand{\attribs}{\mathcal{A}}
\newcommand{\boolattribs}{\mathcal{B}}

\section{Definitions and Problem Formulation}
\label{sec:definitions}

\subsection{Preliminaries}
\label{sse:prelims}

\paragraph{Agents, resources and matching rounds}
Let $\bbx{} = \{x_1, x_2, \ldots, x_n\}$ be a set of~$n$ \emph{agents} and
$\bby{} = \{y_1, y_2, \ldots, y_m\}$ a set of $m$ \emph{resources}. The
objective of the system or the \prin{} is to generate~$k$ matchings $M_1,\ldots,M_k$ of agents to
resources, which we henceforth refer to as a
{\bfseries $\mathbf k$-round matching}. However, each agent has certain
requirements that need to be satisfied. 
Firstly, each agent~$x_i$ specifies a set~$K_i\subseteq\{1,2,\ldots,k\}$
and wants to be matched in~$\rho_i$ rounds from $K_i$. 

\paragraph{Restrictions and Restrictions Graph}
In addition to round preferences, agents may have restrictions
due to which they cannot be matched to certain (or all) resources. We model
this using an~$\agset\rset$ bipartite graph with multiple labels on edges, with
each label representing a restriction.  We refer to this graph as the
\textbf{restrictions graph} and denote it as~$G_R(\agset,\rset,E_R)$, where
$E_R$ is the edge set.  Each edge $e\in E_R$ is associated with a set $\Gamma_e$ of
agent-specific labels representing restrictions.  

\paragraph{Compatibility Graph}
Let~$\calc_i$
denote the set of edge labels or
restrictions associated with agent~$x_i$. Suppose a set $C \subseteq
\calc_i$ of restrictions makes resource~$y_j$ \emph{incompatible}
with~$x_i$, we draw a labeled edge $\{x_i,y_j\}$ with labels in~$C$. We say
that $C$ is the \textbf{restrictions set} for the edge $\{x_i, y_j\}$. To
make~$x_i$ compatible with~$y_j$, \emph{all} the restrictions in~$C$ must be removed. There can be hard constraints due to which~$y_j$ can never
be compatible with~$x_i$ even if all the restrictions are removed, in which
case,~$x_i$ and~$y_j$ are not adjacent in~$G_R$. Let~$E\subseteq E_R$ be
the set of all edges for which the restrictions set is empty. The
corresponding subgraph~$\compgr$ is called the \textbf{compatibility
graph}.  
An example showing a restrictions graph, a compatibility
graph and multi-round matching appears as Figure~\ref{fig:example1}. This example (Figure~\ref{fig:example1})  is motivated by the \ls{} application
considered in this work.  There are lab
members (agents) who need to be assigned rooms (resources) -- one person
per room -- on five days (i.e., $k=5$) with some preferences for space and
days.  Restrictions on each edge are induced by the mismatch between agent
preferences and resource properties.  For example,~$x_1$ requires a big
room, while~$y_2$ is small.  Therefore, we add a restriction (or
label)~$\ttbig_1$ to the edge~$\{x_1,y_2\}$.  The restrictions graph for
this problem is shown followed by a multi-round matching solution.
\begin{figure}[ht]
\centering
\includegraphics[scale=.40]{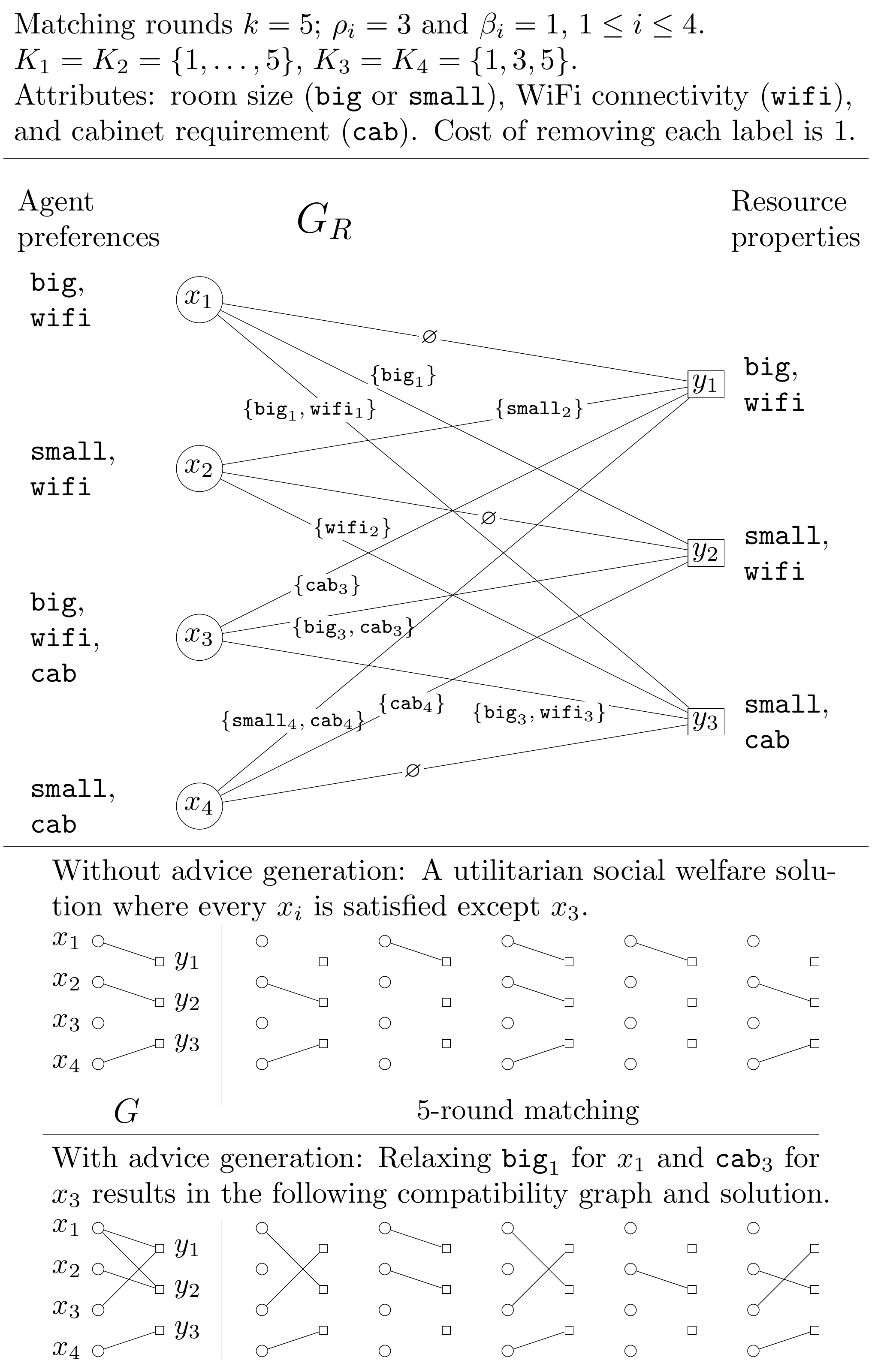}
\caption{An example: A restrictions
graph~$G_R$ is induced by agent preferences and resource attributes. The
first multi-round matching solution corresponds to the compatibility
without the advice generation component. Not all agents are satisfied. The
second solution
satisfies all agents. 
\label{fig:example1}}
\end{figure}

\paragraph{Costs for restrictions removal}
For each restriction~$c_i^t$, there is a positive cost $\psi_i^t$
associated with removing $c_i^t$.  To remove any set of labels, we use an
additive cost model, i.e., for any~$C \subseteq \calc_i$, the cost incurred
by agent~$x_i$ for removing all the restrictions in $C$ is~$\sum_{c_i^t\in C}\psi_i^t$. An agent~$x_i$ is
\textbf{satisfied} if it is matched in~$\rho_i$ rounds belonging to the
set~$K_i$. Now, we formally define the multi-round matching and
advice generation problems.

\subsection{Multi-round Matching}
\label{sse:mrm}
We begin with a definition of the basic multi-round matching problem (\mrm{}) and discuss other versions.

\smallskip
\noindent
\textbf{Multi-Round Matching} (\mrm{})

\noindent
\underline{Instance:}~A compatibility graph $\compgr$, number of
rounds of matching~$k$; for each agent $x_i$, the set of permissible
rounds~$K_i\subseteq\{1,\ldots,k\}$, the number of rounds
$\rho_i\le|K_i|$ in which $x_i$
wants to be matched.

\noindent
\underline{Requirement:}~Find a $k$-round matching 
that satisfies all the agents
(i.e., one that meets the requirements of
all the agents), if one exists.

Since a solution may not exist for a given instance of the $\mrm$ problem,
we consider an extension that uses a \textbf{benefit} (or utility)
function for each agent.  For each agent $x_i$, the benefit function
$\mu_i : \{0, 1, \ldots, \rho_i\} \rightarrow \mathbb{R}^+$, where
$\mathbb{R}^+$ is the set of nonnegative real numbers, gives a benefit
value $\mu_i(\ell)$ when the number of rounds assigned to $x_i$ is $\ell$,
$0 \leq \ell \leq \rho_i$.  If a $k$-round matching assigns $\ell_i \leq
\rho_i$ rounds to agent $x_i$, then the total benefit to all the agents is
given by $\sum_{i=1}^{n} \mu_i(\ell_i)$.  We can now define a more general
version of \mrm{}.

\smallskip

\noindent
\textbf{Multi-Round Matching to Maximize Total Benefit}
(\maxtbmrm)

\noindent
\textbf{\underline{Instance:}}~A compatibility graph $\compgr$, number of
rounds of matching~$k$; for each agent $x_i$, the set of permissible
rounds~$K_i\subseteq\{1,\ldots,k\}$, 
the number of rounds $\rho_i\le|K_i|$ in which
$x_i$ wants to be matched and
the benefit function $\mu_i$.

\noindent
\textbf{\underline{Requirement:}}~Find a $k$-round matching that maximizes the total benefit over all the agents.

As will be shown in Section~\ref{sec:rmmrm_poly}, when the benefit functions satisfy some properties (including monotonicity and diminishing returns), the \maxtbmrm{} problem can be solved efficiently.
One such benefit function allows us to obtain an efficient
algorithm for the \mrm{} problem.
Another benefit function allows us to define the problem
of finding a $k$-round matching that maximizes the number of satisfied
agents (abbreviated as \maxsamrm{}).
However, this benefit function does not satisfy 
the diminishing returns property and so our efficient
algorithm for \maxtbmrm{} cannot be used for this problem.
In fact, we show in Section~\ref{sec:rmmrm_poly}
that the \maxsamrm{} problem is \cnp-hard.


\subsection{Advice Generation}
\label{sec:agmrm}
We define the multi-round matching advice generation problems, namely $\agmrm$ and
$\agmaxsamrm$.

\smallskip
\noindent
\textbf{Advice Generation for Multi-Round Matching}\newline (\agmrm{})

\smallskip
\noindent
\textbf{\underline{Instance:}}~A restrictions graph $G_R(\agset, \rset,
E_R)$, number of rounds of matching~$k$; for each agent $x_i$, the set of
permissible rounds~$K_i\subseteq\{1,\ldots,k\}$, the  number of times
$x_i$ wants to be matched $\rho_i\le|K_i|$, and budget $\beta_i$; for each
label
$c_i^t \in \calc_i$, a cost $\psi_i^t$ of removing 
that label. (Recall that $\calc_i$ is the set of labels
on the edges incident on agent $x_i$ in $G_R$.)

\noindent
\textbf{\underline{Requirement}:}~For each agent $x_i$, is there a subset $\calc'_i \subseteq
\calc_i$ such that the following conditions hold? (i)~The cost of removing
all the labels in $\calc'_i$ for agent $x_i$ is at most $\beta_i$, and (ii)~in the resulting compatibility graph, there exists a
$k$-round matching such that all agents are satisfied.
If so, find such a $k$-round matching.

In Figure~\ref{fig:example1}, the advice generation component is
illustrated in the bottom-most panel. 
In this case, we note that there is a
solution to the $\agmrm$ problem.

In general, since a solution may not exist for a
given $\agmrm$ instance, it is natural to consider the version where the
goal is to maximize the number of satisfied agents.  We denote this version
by \agmaxsamrm{}.

\paragraph{Note:} For convenience, we defined optimization versions of
problems \maxsamrm{}, \agmrm{} and \agmaxsamrm{} above.
In subsequent sections, we show that these problems are \cnp-hard.
It can be seen that the decision versions of these three problems are in \cnp; hence, these versions are \cnp-complete.

\section{Maximizing Total Benefit}
\label{sec:rmmrm_poly} 

\subsection{An Efficient algorithm for \maxtbmrm{}}
In this section, we present our algorithm for the
\maxtbmrm{} problem that finds a $k$-round matching that
maximizes the total benefit over all the agents.
This algorithm requires the benefit function for each agent 
to satisfy some properties. We now specify these
properties.

\smallskip
\noindent
\textbf{Valid benefit functions:}~ 
For each agent $x_i$, $1 \leq i \leq n$,
let~$\reward_i(\ell)$ denote the
non-negative \emph{benefit} that the agent~$x_i$ receives if matched in~$\ell$ rounds, for~$\ell=0,1,2,\ldots,\rho_i$. 
Let~$\increward_i(\ell)=\reward_i(\ell)-
\reward_i(\ell-1)$ for~$\ell=1,2,\ldots,\rho_i$.
We say that the benefit function $\reward_i$
is \textbf{valid} if satisfies
\emph{all} of the following four properties: 
(P1) $\reward_i(0)=0$;
(P2) $\reward_i$ is monotone non-decreasing in~$\ell$;
(P3) $\reward_i$ has the \emph{diminishing returns property}, that is,
$\reward_i(\ell)-\reward_i(\ell-1)\le\reward_i(\ell-1)-\reward_i(\ell-2)$
for~$2\le\ell\le\rho_i$; and 
(P4) $\increward_i(\ell) \leq 1$, for~$\ell=1,2,\ldots,\rho_i$.
Note that~$\reward_i(\cdot)$ satisfies
property~P3
iff~$\increward_i(\cdot)$ is monotone non-increasing in~$\ell$.
Property~P4 can be satisfied by normalizing
the increments in the benefit values.

\smallskip

\noindent
\textbf{Algorithm for \maxtbmrm{}:}~ Recall that the goal of this problem is to find a multi-round matching that maximizes the total benefit over all the agents.
The basic idea behind our algorithm is to reduce the problem to the maximum weighted matching
problem on a larger graph. 
The steps of this construction are shown in
Figure~\ref{fig:alg_rmmrm}. Given a multi-round solution~$\mbsoln$ that assigns $\ell_i$ rounds to agent $x_i$,
$1 \leq i \leq n$,
let~$\totreward(\mbsoln)=\sum_{i=1}^n\reward_i(\ell_i)$ denote the total benefit due to $\mbsoln$.

\begin{figure}[ht]
\rule{\columnwidth}{0.01in}
\small
\begin{enumerate}[leftmargin=*]
\item Given compatibility graph $G(X, Y, E)$, and for each agent~$x_i$,
permissible rounds~$K_i$, the maximum number of rounds~$\rho_i$, and benefit
function~$\reward_i(\cdot)$, create a new edge-weighted bipartite graph
$G'(X', Y', E', w(\cdot))$ as follows.
(The sets $X'$, $Y'$, and $E'$ are all initially empty.)
\begin{description}
\item[Node set $X'$.] For each agent $x_i \in X$ and each $h \in K_i$, add
a node denoted by $x_i^h$ to $X'$.

\item[Node set $Y'$.] For each resource $y_j \in Y$, add~$k$ nodes denoted
by $y_j^h$, where $h = 1, 2, \ldots, k$, to $Y'$. These are called \emph{type-1} resource nodes. 
For every agent~$x_i$, add $\rho_i$ nodes, denoted by~$y^2_{i,p}$, where $p=1, 2, \ldots,\rho_i$; these are called \emph{type-2} resource nodes. 
For every agent~$x_i$, add $|K_i|-\rho_i$ nodes, denoted by
let~$y^3_{i,p}$, where $p=1, 2, \ldots,|K_i|-\rho_i$; these are called \emph{type-3} resource nodes.

\item[Edge set $E'$ and weight~$w$.] For every edge $\{x_i, y_j\} \in E$
and each $h \in K_i$, add edge $\{x_i^h, y_j^h\}$ to $E'$ with
weight~$w(x_i^h,y_j^h)=1$. For each~$h \in K_i$ and~$p=1,\ldots,\rho_i$, add
edge $\{x_i^h, y^2_{i,p}\}$ with
weight~$w(x_i^h,y^2_{i,p})=1-\increward_i(p)$. 
For each~$h\in K_i$
and~$p=1,\ldots,|K_i|-\rho_i$, add edge $\{x_i^h, y^3_{i,p}\}$~with
weight~$w(x_i^h,y^3_{i,p})=k+1$.
\end{description}

\item Find a maximum weighted matching $M^*$ in $G'$.

\item Compute a collection of matchings $\mbsoln^*=\{M_1, M_2, \ldots,
 M_k\}$ for $G$ as follows. 
 The sets $M_1$, $M_2$, $\ldots$, $M_k$ are initially empty.
 For each edge $e=\{x_i^{h},y_j^{h}\} \in M^*$
where~$y_j^h$ is a type-1 resource, add the edge $\{x_i, y_j\}$ to $M_{h}$.
\end{enumerate}
\caption{Steps of \algmaxtbmrm{}.}
\label{fig:alg_rmmrm}
\rule{\columnwidth}{0.01in}
\end{figure}

Now, we establish the following result.
\begin{theorem}\label{thm:alg-correctness}
Algorithm \algmaxtbmrm{} (Figure~\ref{fig:alg_rmmrm})
produces an optimal solution if every 
benefit function~$\reward_i(\cdot)$ is valid.
\end{theorem}
We prove the above theorem through a sequence of 
definitions and lemmas.
\begin{definition} 
A matching $M$ of $G'$ is \textbf{saturated} if~$M$ includes (i)~every node
of $X'$ and (ii)~every node of $Y'$ that represents a type-3 resource.
\end{definition}
\begin{lemma}\label{lem:saturated_matching}
Given any maximum weight matching $M$ for $G'$,  a saturated
maximum weight matching $M'$ for $G'$ with the same weight
as $M$ can be constructed.
\end{lemma}
\noindent
\textbf{Proof sketch:}~ We first argue that every maximum
weight matching of $G'$ includes all type-3 resource nodes
(since each edge incident on those nodes has a large weight).
We then argue that other unmatched nodes of $X'$ can be added
suitably.
\iftoggle{arxiv}
{
For details, see Section~\ref{sup:sec:rmmrm_poly} of
the supplement.
}
{}
\smallskip

Let the quantity $\lambda$ be defined by
{\small
\begin{center}
$\lambda ~=~ \sum_{i=1}^{n} k(|K_i|-\rho_i) +
\sum_{i=1}^{n} \sum_{\ell=1}^{\rho_i}(1-\delta_i(\ell))$.
\end{center}
}
\noindent
We use $\lambda$ throughout the remainder
of this section.
Note that $\lambda$ depends only on the parameters of the problem; 
it does not depend on the algorithm. 
\begin{lemma}\label{lem:matching-to-benefit}
Suppose there is a saturated maximum weight matching~$M$ for~$G'$ with
weight~$W$. Then,  a solution to the \maxtbmrm{} problem instance
with benefit~$W-\lambda$ can be constructed if every benefit function~$\reward_i(\cdot)$ is valid.
\end{lemma}
\begin{proof}
Given~$M$, we compute a collection of matchings $\mbsoln=\{M_1, M_2, \ldots,
 M_k\}$ for $G$ as follows. For each edge $e=\{x_i^{h},y_j^{h}\} \in M$
where~$y_j^h$ is a type-1 resource, add the edge $\{x_i, y_j\}$ to $M_{h}$.
First, we will show that each~$M_h\in\mbsoln$ is a matching in~$G$.
Suppose~$M_h$ is not a matching, then, there exists some~$x_i$
or~$y_j$ with a degree of at least two in~$M_h$. We will prove it for the first
case. The second case is similar. 
Suppose 
agent~$x_i$ is adjacent to two resources~$y_{j'}$ and~$y_{j''}$. This
implies that in~$M$,~$x_i^h$ is adjacent to both~$y_{j'}^h$
and~$y_{j''}^h$, contradicting the fact that~$M$ is a matching in~$G'$.
Hence,~$\mbsoln$ is a valid solution to the \maxtbmrm{} problem.

We now derive an expression for the weight $W$ of $M$.
Since every type-3 resource is matched in~$M$, the corresponding edges
contribute $\sum_{i=1}^nk(|K_i|-\rho_i)$ to $W$. 
For an
agent~$x_i$, let the number of edges to  type-1 resources
be~$\gamma_i\le\rho_i$. Each such edge is  of weight $1$. The
remaining~$\rho_i-\gamma_i$ edges correspond to type-2 resources. Since each
benefit function, $\reward_i$ satisfies the diminishing returns property, the incremental benefits
$\increward_i(\cdot)$ are monotone non-increasing in the number of matchings.
Thus, we may assume without loss of generality  that the top~$\rho_i-\gamma_i$
edges in terms of weight are in 
the maximum weighted matching~$M$. 
This contributes 
$\sum_{\ell=1}^{\rho_i-\gamma_i}\big(1-\increward_i(\rho_i-\ell+1)\big)$ to $W$.
Combining the terms corresponding to type-1 and type-2 
resources, the contribution of agent $x_i$ to $W$  is
\(
\gamma_i +
\sum_{\ell=1}^{\rho_i-\gamma_i}\big(1-\increward_i(\rho_i-\ell+1)\big) = 
\sum_{\ell=1}^{\gamma_i}1 +
\sum_{\ell=\gamma_i}^{\rho_i}\big(1-\increward_i(\ell)\big) 
= \sum_{\ell=1}^{\rho_i}\big(1-\increward_i(\ell)\big) +
\sum_{\ell=1}^{\gamma_i}\increward_i(\ell)
\)\,.
Summing this over all the agents, we get $W$ = $\sum_{i=1}^nk(|K_i|-\rho_i) +
\sum_{i=1}^n\sum_{\ell=1}^{\rho_i}\big(1-\increward_i(\ell)\big) +
\sum_{i=1}^n\sum_{\ell=1}^{\gamma_i}\increward_i(\ell)=\lambda+
\sum_{i=1}^n\reward_i(\gamma_i)$.
Since $\sum_{i=1}^n\reward_i(\gamma_i)$ =
~$\totreward(\mbsoln)$, the lemma follows.
\end{proof}

\begin{lemma}\label{lem:benefit-to-matching}
Suppose there is a solution to the \maxtbmrm{} problem instance with
benefit $Q$. Then, there  a saturated matching $M$ for $G'$ with weight $Q + \lambda$ can be constructed.
\end{lemma}
\noindent
\textbf{Proof idea:}~ This proof uses an analysis
similar to the one used in the proof of 
Lemma~\ref{lem:matching-to-benefit}.
\iftoggle{arxiv}
{
For details, see Section~\ref{sup:sec:rmmrm_poly}
of the supplement.
}
{}
\begin{lemma}\label{lem:opt-benefit-matching}
There is an optimal solution to the MaxTB-MRM problem instance with benefit $Q^*$
if and only if there is a maximum weight saturated matching $M^*$
for $G'$ with weight $Q^* + \lambda$.
\end{lemma}

\noindent
\textbf{Proof idea:}~ This is a consequence of
Lemmas~\ref{lem:matching-to-benefit} and \ref{lem:benefit-to-matching}.
\iftoggle{arxiv}
{
For details see Section~\ref{sup:sec:rmmrm_poly} of the supplement.
}
{}

\smallskip

\noindent
Theorem~\ref{thm:alg-correctness} is a direct consequence of
Lemma~\ref{lem:opt-benefit-matching}.

\smallskip
\noindent
We now estimate the running time of \algmaxtbmrm{}.

\begin{proposition}\label{pro:time-algmaxtbmrm}
Algorithm \algmaxtbmrm{} runs in time 
$O(k^{3/2} (n+|E|)\sqrt{n+m}))$, where $n$ is the 
number of agents, $m$ is the number of resources,
$k$ is the number of rounds and $E$ is the number of
edges in the compatibility graph $G(X, Y, E)$.
\end{proposition}
\iftoggle{arxiv}
{
\noindent
{\it Proof:} See Section~\ref{sup:sec:rmmrm_poly}
of the supplement.
}
{}
\subsection{Maximizing Utilitarian Social Welfare}
\label{sse:max_utilitarian}

Here, we present a valid benefit function that models
utilitarian social welfare.
For each agent $x_i$, let $\mu_i(\ell) = \ell$,
for $\ell{} = 0, 1, \ldots \rho_i$.
It is easy to verify that this is a valid benefit
function.
Hence, the algorithm in Figure~\ref{fig:alg_rmmrm}
can be used to solve the \maxtbmrm{} problem with these benefit
functions.
An optimal solution in this
case maximizes the total number of rounds
assigned to all the agents, with agent $x_i$ assigned at
most $\rho_i$ rounds, $1 \leq i \leq n$.

This benefit function can also be used
to show that the \mrm{} problem (where the goal is to
check whether there is a solution that satisfies all the
agents) can be solved efficiently.

\begin{proposition}\label{pro:mrm_solution}
\mrm{} problem can be solved in polynomial time. 
\end{proposition}

\noindent

\begin{proof}
For each agent $x_i$, let the benefit function $\mu_i$
be defined by $\mu_i(\ell) = \ell$, 
for $0 \leq \ell{} \leq \rho_i$.
We will show that for this benefit function, the
algorithm in Figure~\ref{fig:alg_rmmrm} produces a
solution with total benefit $\sum_{i=1}^{n} \rho_i$
iff there is a solution to the \mrm{} instance.

Suppose there is a solution to the \mrm{} instance. We can
assume (by deleting, if necessary, some rounds in which agents are matched) that 
each agent $x_i$ is assigned exactly $\rho_i$ rounds.
Therefore, the total benefit over all the agents
is $\sum_{i=1}^{n} \rho_i$.
Since the maximum benefit that agent $x_i$ can get is
$\rho_i$, this sum also represents the maximum possible total benefit. So, the total benefit of the solution
produced by the algorithm in Figure~\ref{fig:alg_rmmrm}
is $\sum_{i=1}^{n} \rho_i$.

For the converse, suppose the
maximum benefit produced by the algorithm
is equal to $\sum_{i=1}^{n} \rho_i$.
It can be seen from Figure~\ref{fig:alg_rmmrm} that
for any agent $x_i$, the algorithm assigns at most
$\rho_i$ rounds. (This is due to the type-3 
resource nodes for which all the incident edges have large
weights.)
Since the total benefit is $\sum_{i=1}^{n} \rho_i$, 
each agent $x_i$ must be assigned exactly $\rho_i$ rounds.
Thus, the \mrm{} instance has a solution.
\end{proof}

In subsequent sections, we will refer to the algorithm
for \mrm{} mentioned in the above proof as 
\algmrm.


\newcommand{\ord}{\pi}
\newcommand{\Fg}{F^>}

\subsection{Maximizing Rawlsian Social Welfare}
\label{sec:rawlsian}
Let~$\gamma_i(\mbsoln)$ denote the number of rounds assigned to agent~$x_i$ 
in a multi-round solution~$\mbsoln$. 
The minimum satisfaction ratio
for~$\mbsoln$ is defined as~$\min_{i}\{\gamma_i(\mbsoln)/\rho_i\}$.
Consider the problem of finding a multi-round matching that
\emph{maximizes} the minimum satisfaction ratio.
While the maximization objective of this problem 
seems different from the utilitarian welfare function, 
we have the following result.
\begin{theorem}\label{thm:rawlsian}
There exists a reward function $\mu_i$ for each agent $x_i$ such that maximizing the total benefit under this function
maximizes the Rawlsian social welfare function,
i.e., it maximizes the minimum
satisfaction ratio over all agents. 
Further, this reward function is valid 
and therefore, an optimal solution to the \maxtbmrm{}
problem under this benefit function can
be computed in polynomial time.
\end{theorem}
\noindent
To prove the theorem, we first show how
the benefit function is constructed. 
\begin{enumerate}[leftmargin=*]
\item Let~$F=\big\{k_1/k_2\mid k_1,k_2\in\{1,\ldots,k\} ~\mathrm{and}~ k_1 \leq k_2\big\}\cup\{0\}$.
\item Let~$\ord: F\rightarrow\{1,\ldots,|F|\}$ give the index of
each element in~$F$ when sorted in descending order.
\item For each~$q\in F$, let $\xi(q)=(nk)^{\ord(q)}/(nk)^{|F|}$, where~$n$ is
the number of agents and~$k$ is the number of rounds. Note that
each~$\xi(q)\in(0,1]$.
\item The incremental benefit~$\delta_i(\ell)$ for an agent~$x_i$ for
the~$\ell$th matching is defined
as~$\delta_i(\ell)=\xi\big((\ell-1)/\rho_i\big)$.
Thus, the benefit function $\mu_i$ for agent $x_i$ is
given by $\mu_i(0) = 0$ and 
$\mu_i(\ell) = \mu_i(\ell-1) + \delta_i(\ell)$ for
$1 \leq \ell \leq \rho_i$.
\end{enumerate}
It can be seen that $\mu_i$ satisfies
properties P1, P2, and P4 of a valid benefit
function. We now show that it also satisfies P3,
the diminishing returns property.
\begin{lemma}
For each agent~$x_i$, the incremental benefit~$\delta_i(\ell)$ is monotone
non-increasing in~$\ell$. Therefore,~$\mu_i(\cdot)$ satisfies
the diminishing returns property.
\end{lemma}
\begin{proof}
By definition, for~$\ell=1,\ldots,\rho_i-1$,~$\delta_i(\ell)/\delta_i(\ell+1)=
\xi\big((\ell-1)/\rho_i\big)/\xi\big(\ell/\rho_i\big)\ge nk > 1$, by noting
that~$\pi(\ell/\rho_i)-\pi((\ell+1)/\rho_i)\ge1$. Hence, the lemma holds.
\end{proof}
To complete the proof of Theorem~\ref{thm:rawlsian},
we must show that any solution that maximizes the
total benefit for the above benefit function also maximizes
the minimum satisfaction ratio.
\iftoggle{arxiv}
{
This proof appears in Section~\ref{sup:sec:rmmrm_poly} 
of the supplement.
}
{
This proof appears in~\cite{mround-arxiv}.
}

\subsection{Hardness of \maxsamrm{}}
\label{sse:maxsa-mrm}

Here we consider the \maxsamrm{} problem, where the goal
is to find a $k$-round matching to maximize the number of satisfied agents.
A benefit function that models this problem is as follows.
For each agent $x_i$, let $\mu_i(\ell) = 0$ for
$0 \leq \ell{} \leq \rho_i -1$ and $\mu_i(\rho_i) = 1$.
This function can be seen to satisfy properties P1, P2, and P4 of 
a valid benefit function.
However, it does \emph{not} satisfy P3, the diminishing returns property, since $\delta(\rho_i-1) = 0$ while
$\delta(\rho_i) = 1$.
Thus, this is not a valid benefit function.
This difference is enough to 
change the complexity of the problem.

\begin{theorem}\label{thm:max_sat_agents_npc}
\maxsamrm{} is \cnp-hard even when the number of rounds ($k$) is 3.
\end{theorem}

\noindent
\textbf{Proof idea:}~
We use a reduction from the Minimum Vertex Cover
Problem for Cubic graphs which is known to 
be \cnp-complete~\cite{GareyJohnson79}.
\iftoggle{arxiv}
{
For details, see
Section~\ref{sup:sec:rmmrm_poly} of the supplement. 
}
{}

\section{Advice Generation Problems}
\label{sec:algmsagmrm}

\newcommand{\setc}{\mathcal{Z}}

Here, we first point out that the \agmrm{} and \agmaxsamrm{} problems are \cnp-hard.
We present two solution approaches for the \agmaxsamrm{} problem, namely an
integer linear program
(ILP) formulation and a pruned-search-based optimization using \algmrm{}.

\smallskip
\noindent
\subsection{Complexity Results}
\label{sse:ag_complexity}

\begin{theorem}\label{thm:hardness}
The \agmrm{} problem is NP-hard when there is just one
agent~$x_i$ for which~$\rho_i > 1$.
\end{theorem}

\noindent
\textbf{Proof idea.}~We use a reduction
from the Minimum Set Cover (MSC) problem
which is known to be \cnp-complete
\cite{GareyJohnson79}.
\iftoggle{arxiv}
{
For details, see
Section~\ref{sup:sec:algmsagmrm} of the supplement.
}
{}

Any solution to the \agmrm{} problem must
satisfy \emph{all} the agents.
Thus, the hardness of \agmrm{} also yields
the following:

\begin{corollary}\label{cor:agmaxsa-mrm-hard}
\agmaxsamrm{} is \cnp-hard.
\end{corollary}

\smallskip

\noindent
\subsection{Solution Approaches for
\agmaxsamrm{}}

\noindent
\textbf{(a) ILP Formulation for \agmaxsamrm{}:}~ 
\iftoggle{arxiv}
{
This ILP formulation is presented in section~\ref{sup:sec:algmsagmrm} of
the supplement.
}
{
A complete ILP formulation for solving the problem \agmaxsamrm{} can be found at~\cite{mround-arxiv}. 
}

\smallskip
\noindent
\textbf{(b) \algmrm{} guided optimization:}~
We describe a local search heuristic \heuristicagmrm{} for the \msagmrm{} problem. It consists of
three parts: (1)~identification of relevant candidate sets of relaxations
for each agent and edges in the restrictions graph (and hence the name
\textbf{pruned-search}), (2)~the valuation of an
identified set of relaxations for all agents and (3)~a local search algorithm.

In part~(1), for each agent $x_i$, we first identify the sets
of restrictions 
$\Gamma_i=\{\calc'_i \subseteq \calc_i ~|~$ Cost of removing all the
labels in $\calc'_i$ for agent $x_i$ is $\leq$ $\beta_i\}$ 
to be considered for removal during the search. We choose a subcollection
$\Gamma_i^s\subseteq\Gamma_i$ of sets based on the following criteria:
(a)~only Pareto optimal sets of constraints with respect to the budget are
considered, i.e., if $\calc'_i \in \Gamma_i^s$, then there is no $c_{i}^t
\in \calc_i$ such that the cost of relaxation of $\calc'_i \cup \{c_i^t\}$
is less than or equal to $\beta_i$;  (b)~each $\calc'_i \in \Gamma_i^s$ is associated
with a set of edges~$E_i'$ that will be added 
to the compatibility graph
if~$\calc'_i$ is relaxed; if~$\calc'_i$ and~$\calc''_i$ are such
that~$E_i''\subseteq E_i'$, then, we retain only~$\calc'_i$; 
(c)~if two sets of constraints cause the addition of the same
set of edges to the compatibility graph, only one of them is considered.

In part~(2), given a collection of relaxation sets, one for each
unsatisfied agent, we apply \algmrm{} from Section~\ref{sec:rmmrm_poly} on
the induced compatibility graph. The number of agents that are satisfied in
the $k$-round solution obtained serves as the valuation for the collection
of relaxation sets.  In part~(3), we use simulated annealing local search
algorithm \cite{AK-1989} for the actual search. 
\iftoggle{arxiv}
{
The details of this implementation are in Section~\ref{sup:sec:experiments}
of the supplement.
}
{}


\section{Experimental Results}
\label{sec:experiments}

We applied the algorithms developed in this work to several datasets.  
Evaluations were based on  relevant benefit
functions and computing time. We used at least~10 replicates for each
experiment when needed (e.g., random assignment of attributes and while
using \heuristicagmrm{}, which is a stochastic optimization algorithm).
Error bars are provided in such instances.
\iftoggle{arxiv}
{
Implementation details are in Section~\ref{sup:sec:experiments} of the
supplement.
}
{}
We considered two novel real-world applications
(discussed below). Here, we
describe the attributes of the datasets; the data itself and the code for running the experiments are
available at \url{https://github.com/yohayt/Resource-Sharing-Through-Multi-Round-Matchings}. 

\paragraph{\ls{} application.} We conducted a back-to-the-lab desk sharing
study in the AI lab at a university\footnote{\label{note1}Bar-Ilan University, Ramat Gan, Israel.} to facilitate lab
personnel intending to return to in-person work 
during the COVID-19 epidemic. A
survey was used to collect the preferences of lab members. The lab
has 14 offices and 31 members.
Six of the offices can accommodate two students each and the remaining eight can accommodate one student each. 
The agent restrictions are as follows:
vicinity to the advisor's room, presence of cabinet, WiFi
connectivity, ambient noise, vicinity to the kitchen, room size, and vicinity
to the bathroom. All agent preferences are generated from the survey as a
function of the affinity (from 1~(low) to 5~(high)) they provide for each
attribute. These affinities are also used as costs for attribute removal.

\paragraph{\cc{} dataset.} This dataset comes from a university\textsuperscript{\ref{note1}} for the
year~2018--2019. There are~$144$ classrooms and~$153$ courses. In the
experiments, we focused on two hours on Tuesday and used all the courses
that are scheduled in this time slot and all available classes. There were
142 courses and all the 144 classrooms were available for them.  Each
classroom has two attributes: its \emph{capacity} and the \emph{region} to
which it belongs. Although the problem of assigning classrooms to courses
is quite common~\cite{phillips2015integer}, we did not find any publicly
available dataset that could be used here. The number of rounds of
participation is chosen randomly between~$1$ and~$3$ as this data was not
available. Also, to generate~$K_i$,~$\rho_i$ rounds are sampled uniformly
followed by choosing the remaining days with probability~$0.5$. 
We used~10 replicates for each experiment.


The generation of the restrictions and compatibility graphs from the agent
preferences and resource attributes is described in
detail in
\iftoggle{arxiv}
{Section~\ref{sec:experiments} of the supplement.}
{~\cite{mround-arxiv}.}

\begin{figure}[htb]
\centering
\includegraphics[width=.5\columnwidth]{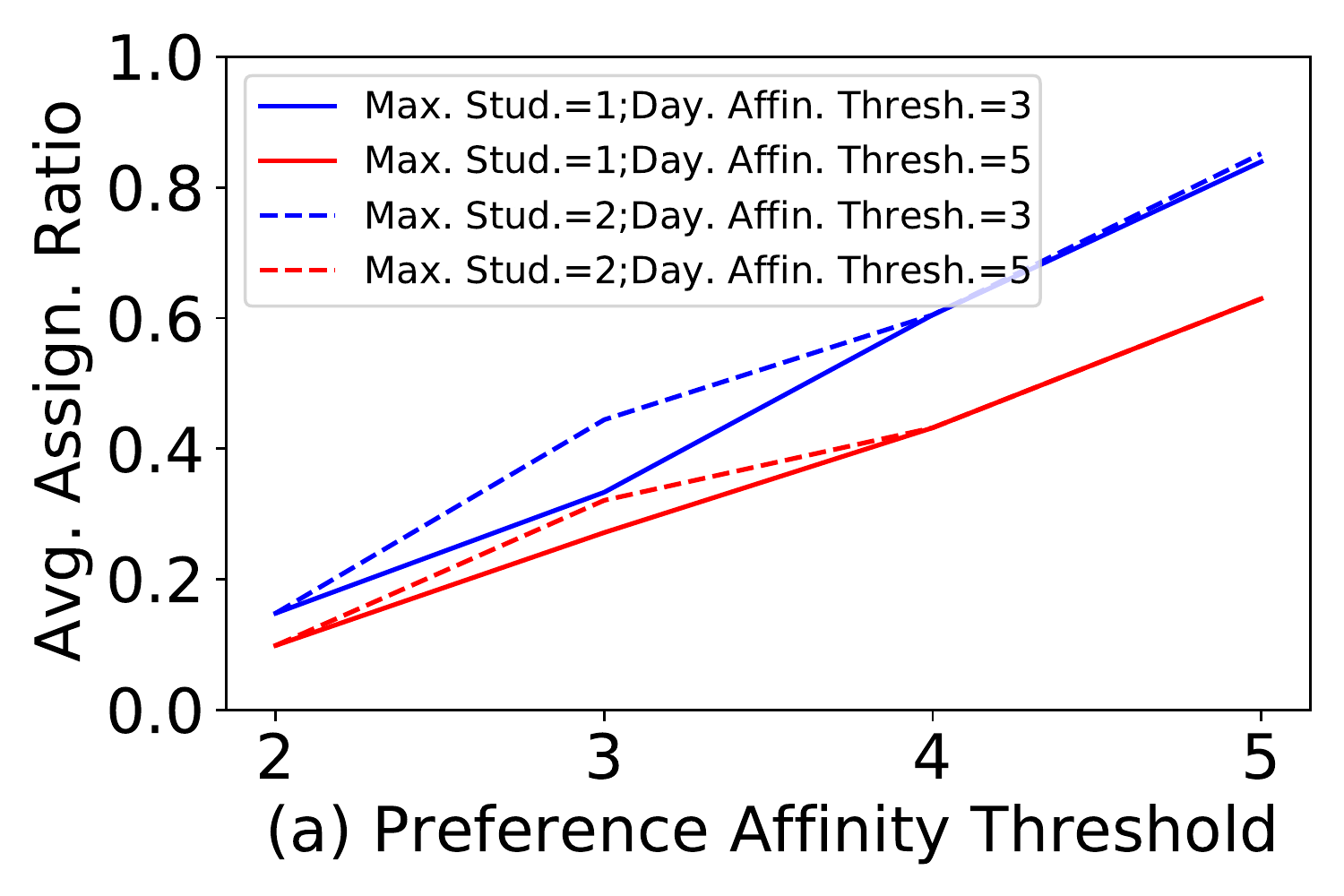}
\includegraphics[width=.48\columnwidth]{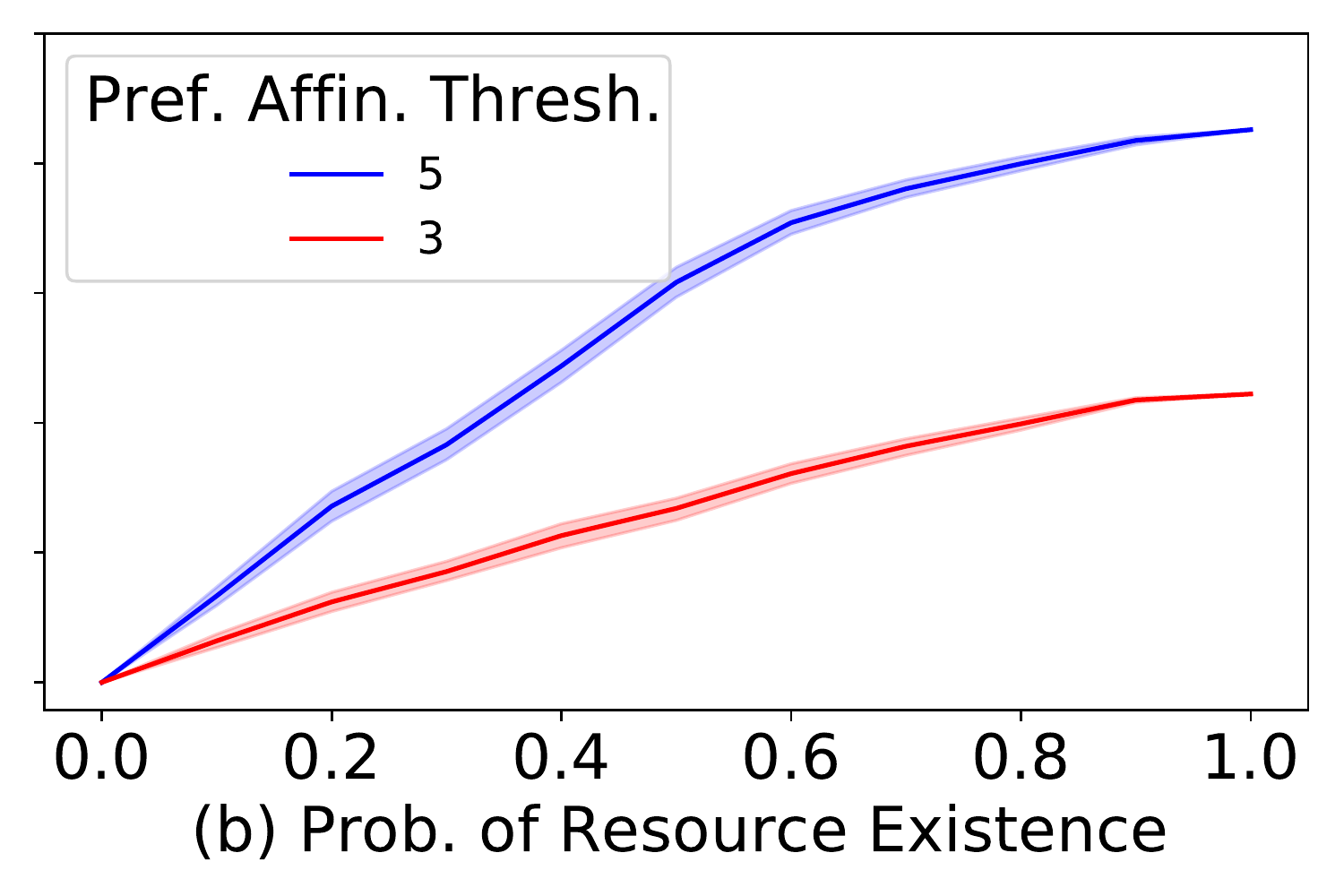}
\includegraphics[width=.5\columnwidth]{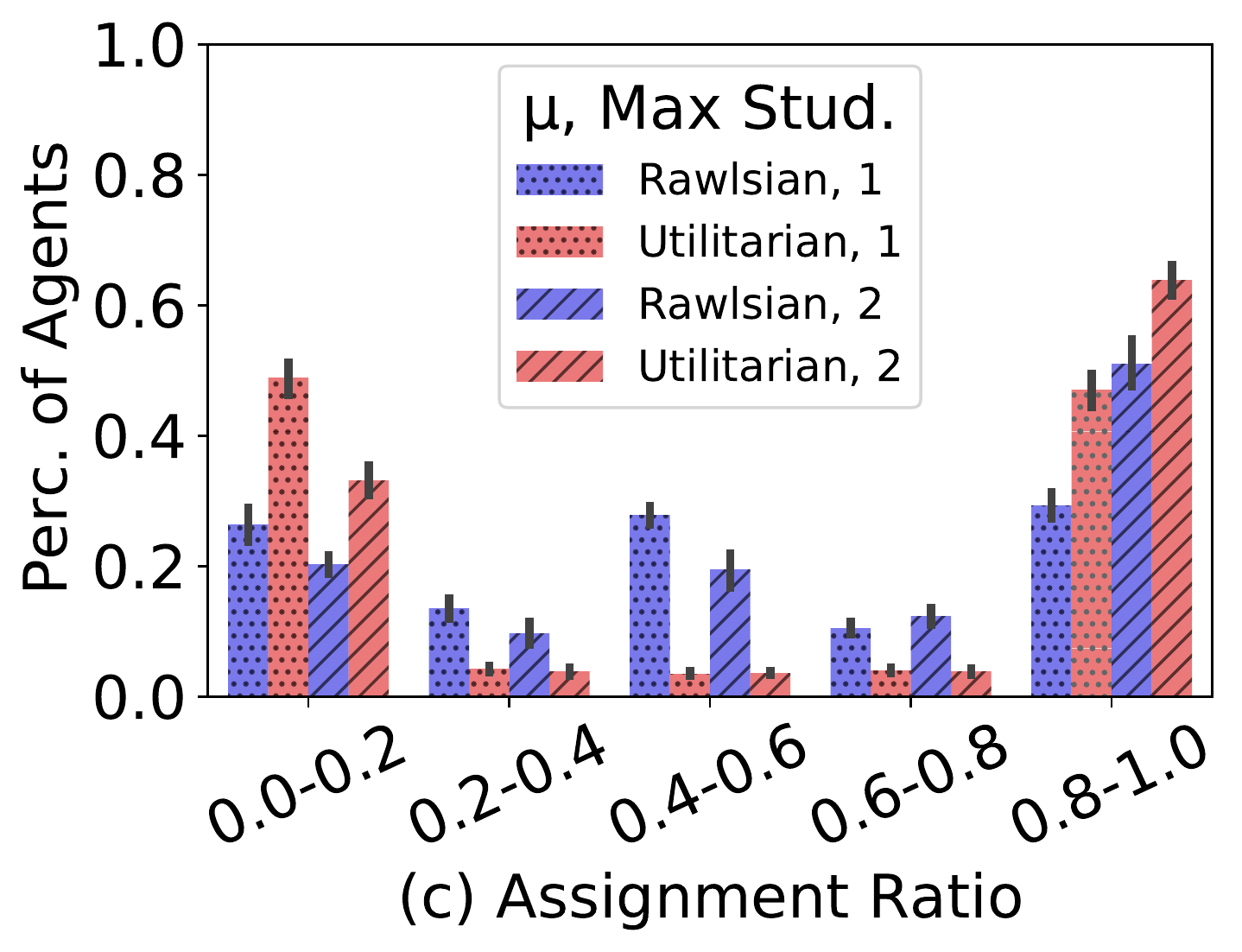}
\includegraphics[width=.49\columnwidth]{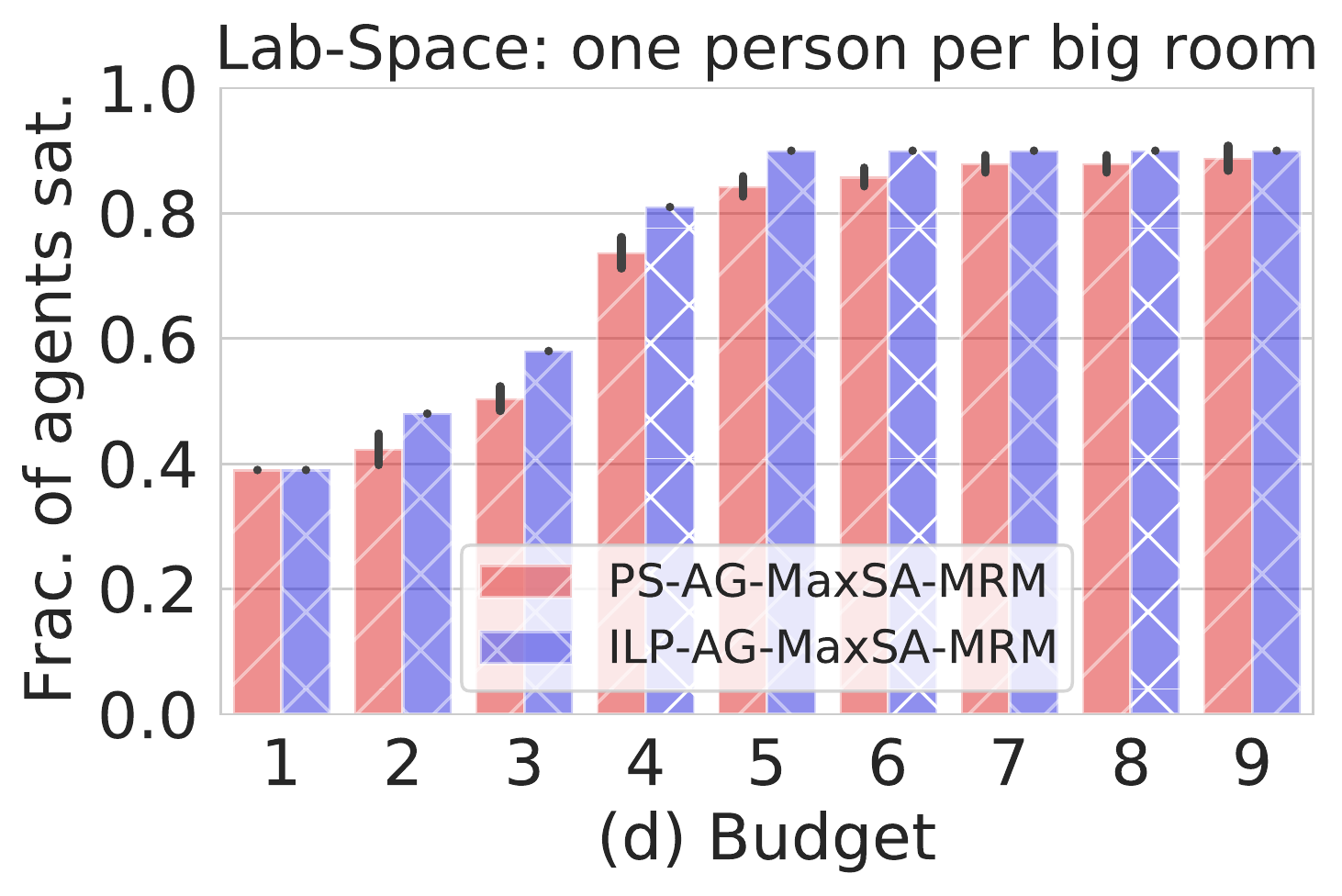}
\caption{\ls{} experiment results: (a)~Average assignment ratio for
different preference affinity thresholds; (b)~Average assignment ratio when
the number of resources is restricted; (c)~Comparison of different welfare
functions for resource probability~$0.5$, and preference affinity
and day affinity thresholds of~$5$ and~$2$ respectively; and (d)~Evaluation
of the advice generation algorithm for increasing budget.
\iftoggle{arxiv}
{
(See figure~\ref{fig:ag2students} for the case with two
people per big room.)
}
{(Plots for the case with two
people per big room are in~\cite{mround-arxiv}.) 
}
\label{fig:ls_total}
}

\end{figure}
\paragraph{\ls{} multi-round matching.} Through this study, the head of the lab
(\prin{}) wanted answers to the following questions: Is it possible to
accommodate only one person per office? Is it possible to maximize the
number of relevant assignments? Is it possible to achieve
all this by satisfying as many preferences of lab members as possible?  The
number of rounds is~$5$ (one for each weekday).  We defined a preference
affinity threshold. If the agent's survey score for affinity (1-5) is
at least as much as the threshold, then we add this preference as a constraint.
Note that the compatibility graph generated for threshold~$\tau$ is a
subgraph of the one generated for~$\tau+1$. The permissible sets of rounds~$K_i$ are
generated the same way. A day affinity threshold is set. If the agent's
preference for a weekday is lower than this threshold, then we do not add
it to~$K_i$. In Figure~\ref{fig:ls_total}(a), we have results for maximizing the
utilitarian social welfare for
various compatibility graphs generated by adding constraints based on the
agent's affinity to each preference. We see that adhering strictly to the
preferences of agents gives a very low average assignment ratio
$\gamma_i/\rho_i$, where $\gamma_i$ is the number of
rounds assigned to agent $x_i$ (who requested $\rho_i$ rounds). We considered two scenarios for large rooms: one or
two students. We note that accommodating two students significantly
improves the average assignment ratio. Adhering to the preferred days of
students does not seem to have many costs, particularly when the preference
affinity threshold is high. In Figure~\ref{fig:ls_total}(b), we
demonstrate how critical the current set of resources is for the functioning of
the lab. We withheld only a portion of the resources (by sampling)
to satisfy agent preferences. However, with just 60\% of the resources, it is
possible to get an average assignment ratio of~$0.6$, albeit by ignoring
most of the agent preferences. In Figure~\ref{fig:ls_total}(c), we compare
utilitarian social welfare with Rawlsian social welfare. We note that in
the utilitarian welfare, many agents end up with a lower assignment ratio
compared to Rawlsian, while the total number of matchings across rounds is
the same. 
\iftoggle{arxiv}
{
(A theorem in this regard is presented in 
Section~\ref{sup:sec:gen_util_theorem}
of the supplement; it states
that for valid and monotone strictly increasing benefit functions, an optimal solution for total benefit also optimizes the total
number of rounds assigned to agents.)
}
{
(A theorem in this regard is presented in 
~\cite{mround-arxiv}.
It states
that for valid and monotone strictly increasing benefit functions, an optimal solution for total benefit also optimizes the total
number of rounds assigned to agents.)
}
Hence, from the perspective
of fairness in agent satisfaction, the Rawlsian reward function performs
better.

\begin{figure}[htb]
\centering
\includegraphics[width=.47\columnwidth]{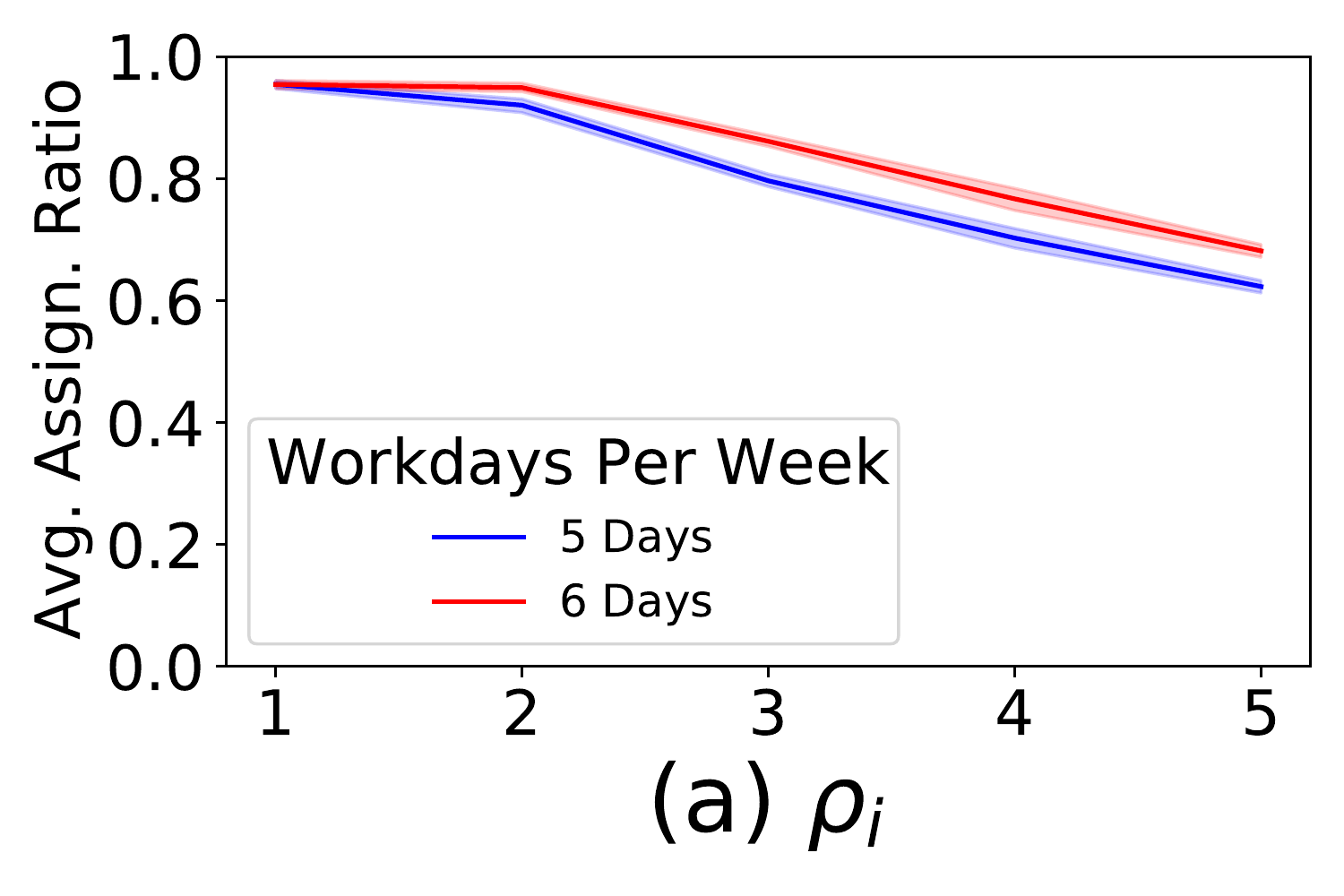}
\includegraphics[width=.47\columnwidth]{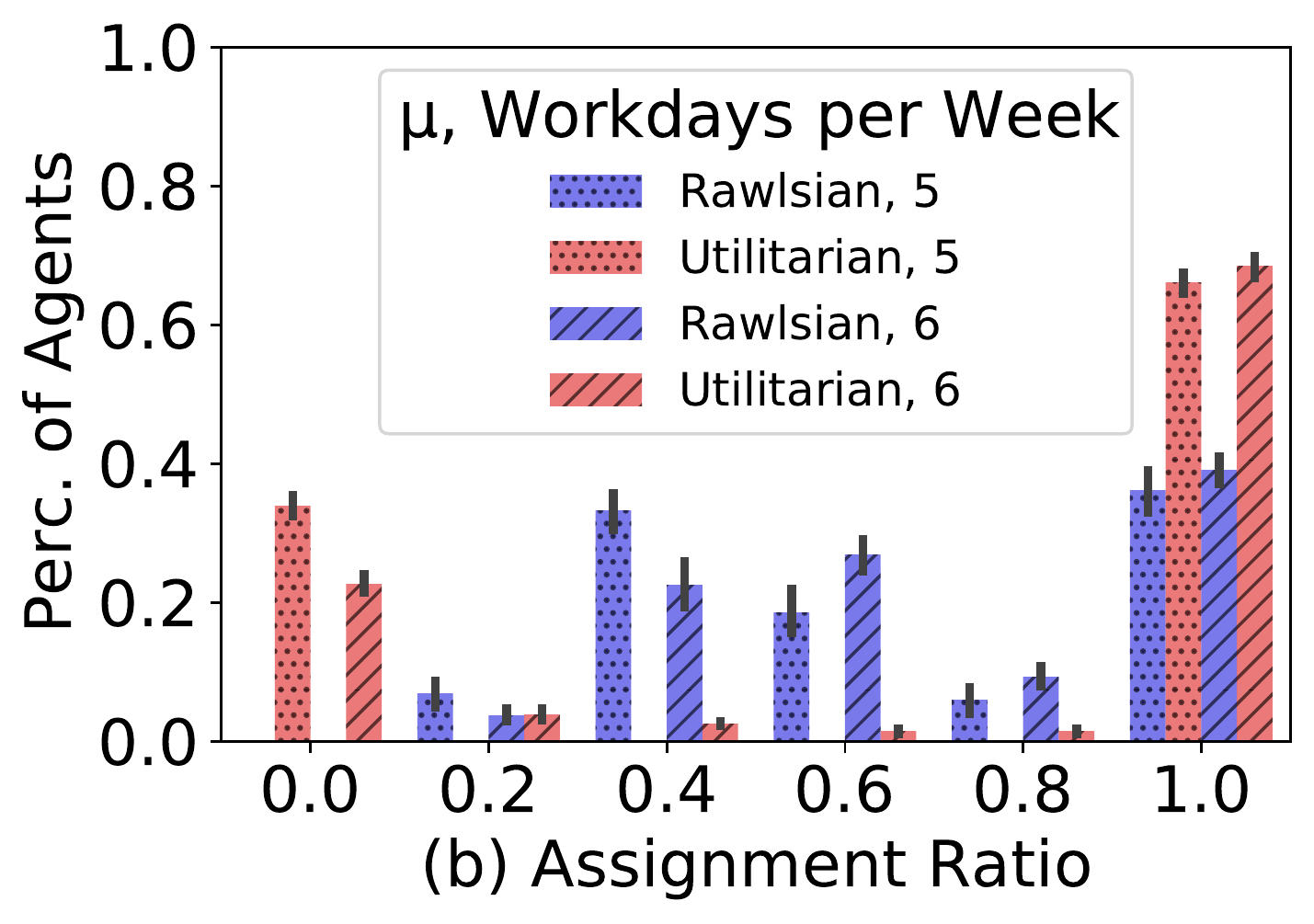}
\includegraphics[width=.47\columnwidth]{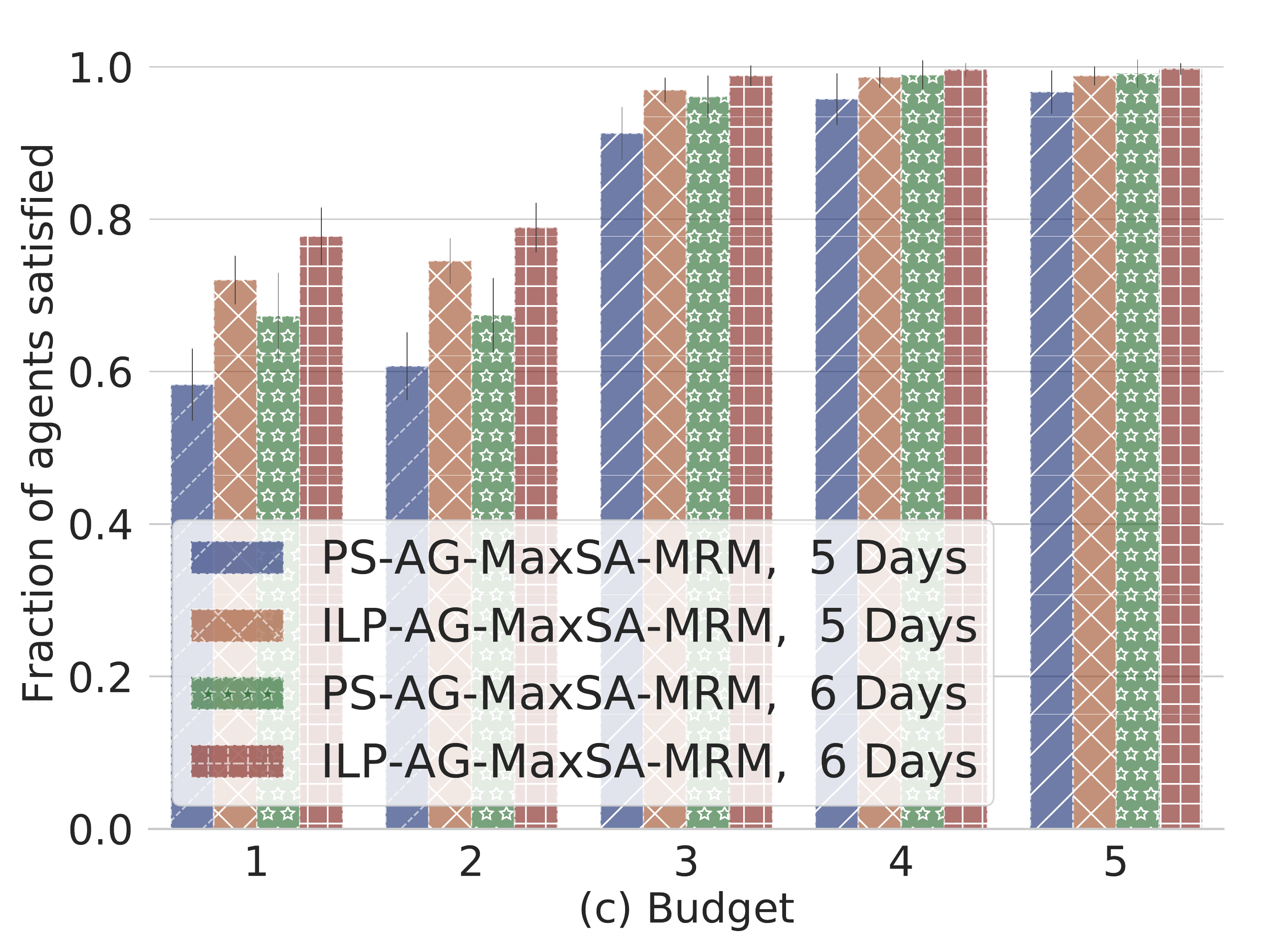}
\includegraphics[width=.47\columnwidth]{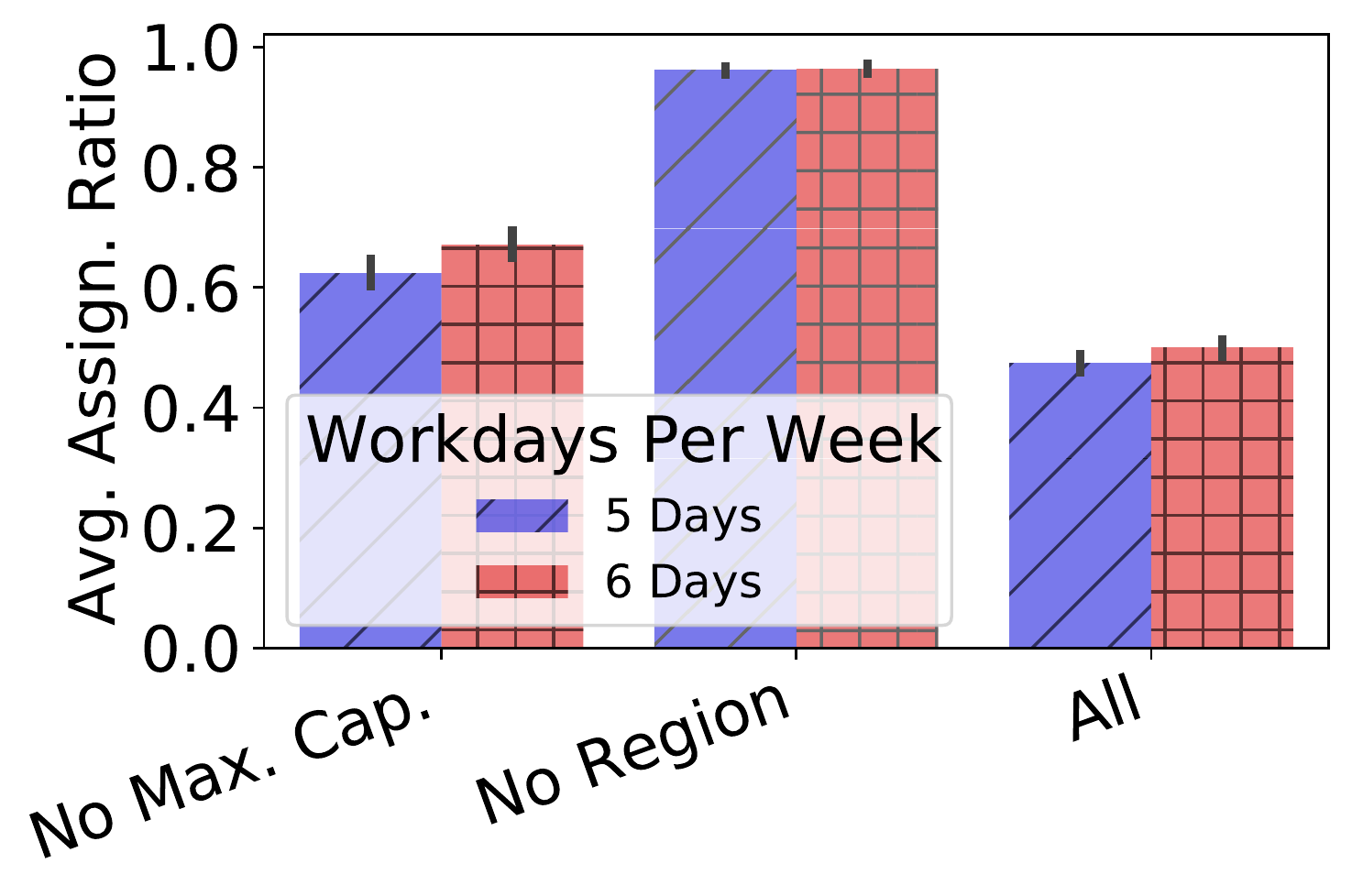}
\caption{\cc{} experimental results: (a)~Average assignment ratio vs. number
    of rounds required per agent; (b)~Comparison of different welfare
    functions. (c)~Evaluation of advice generation algorithm for increasing
budget. (d)~Importance of attributes.
\label{fig:cc}}
\end{figure}

\paragraph{\ls{} advice generation.}
From the results in Figure~\ref{fig:ls_total}(d), we observe
that there is a sharp increase in the number of agents satisfied for a
budget of~4. In the data, there is a strong preference for WiFi (an average
score of~4.3 out of~5), while more than~40\% of the rooms have poor
connectivity. When most of the agents relax this preference, they have
access to many rooms. This partially explains the sharp increase in the
number of matchings. We also observe that the performance of
\heuristicagmrm{} is close to that of optimum (given by \ilpmsagmrm). 


\paragraph{\cc{} multi-round matching.} Here, we considered two scenarios:
having classes (i)~five days a week and (ii)~six days a week. A six-day
week can accommodate more courses while on the other hand, facilities have
to be kept open for an additional day. The results of maximizing the utilitarian 
multi-round matching
are in Figure~\ref{fig:cc}(a). We observe that as~$\rho_i$ is increased,
the average assignment ratio decreases slowly. The difference between five and six days a week is not
significant. Figure~\ref{fig:cc}(b) compares utilitarian and Rawlsian reward
functions. Unlike the \ls{} results, here we clearly see that the minimum
assignment ratio is higher in the former case. 

\paragraph{\cc{} advice generation} In Figure~\ref{fig:cc}(c) we
observe that only for low budgets, the difference between five-day
week and six-day week solutions is significant.  In the same regime, we see
a significant difference between \heuristicagmrm{} and \ilpmsagmrm.  By
inspecting the solution sets, we observed that relaxing the region
attribute is most effective in increasing the number of satisfied agents, while
minimum capacity is least effective. To study the importance of different
attributes, we conducted experiments where a single chosen attribute was
omitted from the restrictions list.  Figure~\ref{fig:cc}(d) indicates that
the region attribute has the most impact on the assignments
ratio.
\paragraph{Scaling to larger networks.}
To analyze the performance of the of the advice generation algorithms with
respect to the size of the network, we experimented with complete bipartite
graphs of various sizes. 
\iftoggle{arxiv}
{The results are in Section~\ref{sup:sec:experiments} of the
supplement.}
{The results are in~\cite{mround-arxiv}.} 
In these experiments, we note that
not only is the solution obtained using the \heuristicagmrm{} close to
\ilpmsagmrm, it is orders of magnitude faster.

\section{Directions for Future Work}
\label{sec:concl}
We conclude by mentioning a few directions for future work. One direction
is to consider other models for producing compatibility graphs from
restrictions graphs.  (For example, an edge may be added to the
compatibility graph if at least one label on that edge is removed.) It will
also be interesting to consider the multi-round matching problems in an
online setting, where the set of agents varies over time and compatibility
is round-dependent.  Finally, it is of interest to investigate
approximation algorithms with provable performance guarantees for
\maxsamrm{} and \agmaxsamrm{}.


\clearpage
\section{ Acknowledgments}

We thank the AAAI 2023 reviewers for their 
feedback.
This research has been partially supported by the Israel Science Foundation under grant 1958/20, the
EU Project TAILOR under grant952215,
University of Virginia Strategic
Investment Fund award number SIF160,
and the US National Science Foundation grant
OAC-1916805 (CINES).


\bibliography{refs}

\begin{thebibliography}{40}
\providecommand{\natexlab}[1]{#1}

\bibitem[{Aarts and Korst(1989)}]{AK-1989}
Aarts, E.; and Korst, J. 1989.
\newblock \emph{Simulated Annealing and Boltzmann Machines: A Stochastic
  Approach to Combinatorial Optimization and Neural Computing}.
\newblock New York, NY: Wiley.

\bibitem[{Anshelevich et~al.(2013)Anshelevich, Chhabra, Das, and
  Gerrior}]{anshelevich2013social}
Anshelevich, E.; Chhabra, M.; Das, S.; and Gerrior, M. 2013.
\newblock On the social welfare of mechanisms for repeated batch matching.
\newblock In \emph{AAAI}.

\bibitem[{Babaei, Karimpour, and Hadidi(2015)}]{babaei2015survey}
Babaei, H.; Karimpour, J.; and Hadidi, A. 2015.
\newblock A survey of approaches for university course timetabling problem.
\newblock \emph{Computers \& Industrial Engineering}, 86: 43--59.

\bibitem[{Cai and Khan(2010)}]{cai2010common}
Cai, H.; and Khan, S. 2010.
\newblock The common first year studio in a hot-desking age: An explorative
  study on the studio environment and learning.
\newblock \emph{J. for Education in the Built Environment}, 5(2): 39--64.

\bibitem[{Caragiannis and Narang(2022)}]{caragiannis2022repeatedly}
Caragiannis, I.; and Narang, S. 2022.
\newblock Repeatedly Matching Items to Agents Fairly and Efficiently.
\newblock \emph{arXiv preprint arXiv:2207.01589}.

\bibitem[{Chevaleyre et~al.(2006)Chevaleyre, Dunne, Endriss, Lang,
  Lema{\^{\i}}tre, Maudet, Padget, Phelps, Rodr{\'{\i}}guez{-}Aguilar, and
  Sousa}]{Chevaleyre-etal-2006}
Chevaleyre, Y.; Dunne, P.~E.; Endriss, U.; Lang, J.; Lema{\^{\i}}tre, M.;
  Maudet, N.; Padget, J.~A.; Phelps, S.; Rodr{\'{\i}}guez{-}Aguilar, J.~A.; and
  Sousa, P. 2006.
\newblock Issues in Multiagent Resource Allocation.
\newblock \emph{Informatica (Slovenia)}, 30(1): 3--31.

\bibitem[{Cormen et~al.(2009)Cormen, Leiserson, Rivest, and Stein}]{CLRS-2009}
Cormen, T.~H.; Leiserson, C.~E.; Rivest, R.~L.; and Stein, C. 2009.
\newblock \emph{Introduction to Algorithms}.
\newblock Cambridge, MA: MIT Press and McGraw-Hill.

\bibitem[{{District Management Group}(2020)}]{dmg2020}
{District Management Group}. 2020.
\newblock {COVID-19: Creating Elementary School Schedules to Support Social
  Distancing DMGroup Research Briefs}.
\newblock
  \url{https://f.hubspotusercontent00.net/hubfs/3412255/Resources\%20Page\%20Files\%20-\%20Public/DMGroup-School-Restart-Research-Brief-Elementary-School-Schedules-and-Social-Distancing_2020-7.pdf}.

\bibitem[{Dolgov and Durfee(2006)}]{Dolgov-etal-2006}
Dolgov, D.~A.; and Durfee, E.~H. 2006.
\newblock Resource Allocation Among Agents with {MDP}-Induced Preferences.
\newblock \emph{J. Artif. Intell. Res.}, 27: 505--549.

\bibitem[{Eley(2006)}]{eley2006ant}
Eley, M. 2006.
\newblock Ant algorithms for the exam timetabling problem.
\newblock In \emph{International Conference on the Practice and Theory of
  Automated Timetabling}, 364--382. Springer.

\bibitem[{{Enriching Students}(2020)}]{enriching2020}
{Enriching Students}. 2020.
\newblock {School Schedules During COVID-19–A Guide to assist schools}.
\newblock
  \url{https://www.enrichingstudents.com/school-schedules-during-covid-19/}.

\bibitem[{Felfernig et~al.(2011)Felfernig, Friedrich, Jannach, and
  Zanker}]{Felfernig-etal-2011}
Felfernig, A.; Friedrich, G.; Jannach, D.; and Zanker, M. 2011.
\newblock Developing Constraint-based Recommenders.
\newblock In \emph{Recommender Systems Handbook}, 187--215. Springer.

\bibitem[{Fisanotti, Carrascosa, and Romero(2012)}]{simpleai}
Fisanotti, J.~P.; Carrascosa, R.; and Romero, S. 2012.
\newblock Simple AI.

\bibitem[{Fonseca, Santos, and Carrano(2016)}]{fonseca2016integrating}
Fonseca, G.~H.; Santos, H.~G.; and Carrano, E.~G. 2016.
\newblock Integrating matheuristics and metaheuristics for timetabling.
\newblock \emph{Computers \& Operations Research}, 74: 108--117.

\bibitem[{Garey and Johnson(1979)}]{GareyJohnson79}
Garey, M.~R.; and Johnson, D.~S. 1979.
\newblock \emph{{Computers and Intractability: A Guide to the Theory of
  NP-completeness}}.
\newblock W. H. Freeman and Co.

\bibitem[{Gilbert(2018)}]{Gilbert-2018}
Gilbert, F. 2018.
\newblock A Guide to Sharing Farm Equipment.
\newblock
  \url{https://projects.sare.org/wp-content/uploads/Sharing-Guide-2018-\_-Web.pdf}.

\bibitem[{Gollapudi, Kollias, and Plaut(2020)}]{gollapudi2020almost}
Gollapudi, S.; Kollias, K.; and Plaut, B. 2020.
\newblock Almost envy-free repeated matching in two-sided markets.
\newblock In \emph{Int. Conf. on Web and Internet Economics}.

\bibitem[{Gorodetski, Karsaev, and Konushy(2003)}]{Gorodetski-etal-2003}
Gorodetski, V.~I.; Karsaev, O.; and Konushy, V. 2003.
\newblock Multi-agent System for Resource Allocation and Scheduling.
\newblock In \emph{Proc. CEEMAS}, 236--246.

\bibitem[{{Gurobi Optimization, LLC}(2021)}]{gurobi}
{Gurobi Optimization, LLC}. 2021.
\newblock {Gurobi Optimizer Reference Manual}.

\bibitem[{Hagberg, Swart, and Chult(2008)}]{hagberg2008exploring}
Hagberg, A.; Swart, P.; and Chult, D.~S. 2008.
\newblock Exploring network structure, dynamics, and function using NetworkX.
\newblock Technical report, Los Alamos National Lab.(LANL), Los Alamos, NM
  (United States).

\bibitem[{Karamanis, Anastasiadis, and
  Angeloudis(2020)}]{karamanis2020assignment}
Karamanis, R.; Anastasiadis, E.; and Angeloudis, M., P.~Stettler. 2020.
\newblock Assignment and pricing of shared rides in ride-sourcing using
  combinatorial double auctions.
\newblock \emph{IITS}.

\bibitem[{Kucharski and Cats(2020)}]{kucharski2020exact}
Kucharski, R.; and Cats, O. 2020.
\newblock Exact matching of attractive shared rides (ExMAS) for system-wide
  strategic evaluations.
\newblock \emph{Transportation Research Part B: Methodological}, 139: 285--310.

\bibitem[{Leite, Mel{\'\i}cio, and Rosa(2019)}]{leite2019fast}
Leite, N.; Mel{\'\i}cio, F.; and Rosa, A.~C. 2019.
\newblock A fast simulated annealing algorithm for the examination timetabling
  problem.
\newblock \emph{Expert Systems with Applications}, 122: 137--151.

\bibitem[{Liu(2020)}]{liu2020stability}
Liu, C. 2020.
\newblock Stability in repeated matching markets.
\newblock \emph{arXiv preprint arXiv:2007.03794}.

\bibitem[{Nguyen, Roos, and Rothe(2013)}]{nguyen2013survey}
Nguyen, T.~T.; Roos, M.; and Rothe, J. 2013.
\newblock A survey of approximability and inapproximability results for social
  welfare optimization in multiagent resource allocation.
\newblock \emph{AMAI}, 68(1-3): 65--90.

\bibitem[{Parker(2020)}]{parker2020covid}
Parker, L.~D. 2020.
\newblock The COVID-19 office in transition: cost, efficiency and the social
  responsibility business case.
\newblock \emph{Accounting, Auditing \& Accountability J.}

\bibitem[{Phillips et~al.(2015)Phillips, Waterer, Ehrgott, and
  Ryan}]{phillips2015integer}
Phillips, A.~E.; Waterer, H.; Ehrgott, M.; and Ryan, D.~M. 2015.
\newblock Integer programming methods for large-scale practical classroom
  assignment problems.
\newblock \emph{Computers \& Operations Research}, 53: 42--53.

\bibitem[{Rakhra and Singh(2020)}]{RS-2020}
Rakhra, M.; and Singh, R. 2020.
\newblock Internet Based Resource Sharing Platform development For Agriculture
  Machinery and Tools in {Punjab, India}.
\newblock In \emph{Proc. 8th International Conference on Reliability, Infocom
  Technologies and Optimization (Trends and Future Directions) (ICRITO)},
  636--642.

\bibitem[{Rawls(1999)}]{rawls1999theory}
Rawls, J. 1999.
\newblock \emph{A Theory of Justice}.
\newblock Harvard University Press.

\bibitem[{Stark(2020)}]{Stark-2020}
Stark, O. 2020.
\newblock An economics-based rationale for the {R}awlsian social welfare
  program.
\newblock University of T\"{u}bingen Working Papers in Business and Economics,
  No. 137.

\bibitem[{S{\"u}hr et~al.(2019)S{\"u}hr, Biega, Zehlike, Gummadi, and
  Chakraborty}]{suhr2019two}
S{\"u}hr, T.; Biega, A.~J.; Zehlike, M.; Gummadi, K.~P.; and Chakraborty, A.
  2019.
\newblock Two-sided fairness for repeated matchings in two-sided markets: A
  case study of a ride-hailing platform.
\newblock In \emph{Proceedings of the 25th ACM SIGKDD International Conference
  on Knowledge Discovery \& Data Mining}, 3082--3092.

\bibitem[{Tan et~al.(2021)Tan, Goh, Kendall, and Sabar}]{tan2021survey}
Tan, J.~S.; Goh, S.~L.; Kendall, G.; and Sabar, N.~R. 2021.
\newblock A survey of the state-of-the-art of optimisation methodologies in
  school timetabling problems.
\newblock \emph{Expert Systems with Applications}, 165: 113943.

\bibitem[{Trabelsi et~al.(2022{\natexlab{a}})Trabelsi, Adiga, Kraus, and
  Ravi}]{trabelsi2022maximizing}
Trabelsi, Y.; Adiga, A.; Kraus, S.; and Ravi, S. 2022{\natexlab{a}}.
\newblock Maximizing Resource Allocation Likelihood with Minimum Compromise.
\newblock In \emph{Proceedings of the 21st International Conference on
  Autonomous Agents and Multiagent Systems}, 1738--1740.

\bibitem[{Trabelsi et~al.(2022{\natexlab{b}})Trabelsi, Adiga, Kraus, and
  Ravi}]{trabelsi2022maximizing2}
Trabelsi, Y.; Adiga, A.; Kraus, S.; and Ravi, S. 2022{\natexlab{b}}.
\newblock Resource Allocation to Agents with Restrictions: Maximizing
  Likelihood with Minimum Compromise.
\newblock In \emph{Proc. European Conference on Multi-Agent Systems (EUMAS)}.

\bibitem[{Varone and Beffa(2019)}]{varone2019dataset}
Varone, S.; and Beffa, C. 2019.
\newblock Dataset on a problem of assigning activities to children, with
  various optimization constraints.
\newblock \emph{Data in brief}, 25: 104168.

\bibitem[{Viner(1949)}]{viner1949bentham}
Viner, J. 1949.
\newblock Bentham and JS Mill: The utilitarian background.
\newblock \emph{The American Economic Review}, 39(2): 360--382.

\bibitem[{Wong(2021)}]{wong2021self}
Wong, M. 2021.
\newblock Self-scheduler for dental students booking consultations with faculty
  during the COVID-19 pandemic.
\newblock \emph{Journal of Dental Education}.

\bibitem[{Zahedi, Sengupta, and Kambhampati(2020)}]{Zahedi-etal-2020}
Zahedi, Z.; Sengupta, S.; and Kambhampati, S. 2020.
\newblock '{Why} not give this work to them?' {E}xplaining {AI}-Moderated
  Task-Allocation Outcomes using Negotiation Trees.
\newblock \emph{CoRR}, abs/2002.01640.

\bibitem[{Zanker, Jessenitschnig, and Schmid(2010)}]{Zanker-etal-2010}
Zanker, M.; Jessenitschnig, M.; and Schmid, W. 2010.
\newblock Preference reasoning with soft constraints in constraint-based
  recommender systems.
\newblock \emph{Constraints An Int. J.}, 15(4): 574--595.

\bibitem[{Zhou and Han(2019)}]{Zhou-Han-2019}
Zhou, W.; and Han, W. 2019.
\newblock Personalized Recommendation via User Preference Matching.
\newblock \emph{Information Processing and Management}, 56(3): 955--968.

\end{thebibliography}
\iftoggle{arxiv}
{
\clearpage
\appendix
\onecolumn
\begin{center}
\fbox{{\Large\textbf{Technical Supplement}}}
\end{center}

\bigskip\bigskip

\noindent
\textbf{Paper title:}~  Resource Sharing Through Multi-Round Matchings 

\medskip


\section{Additional Material for Section~\ref{sec:rmmrm_poly}}
\label{sup:sec:rmmrm_poly}

\medskip


\subsection{Statement and Proof of Lemma~\ref{lem:saturated_matching}}

\medskip

\noindent
\textbf{Statement of Lemma~\ref{lem:saturated_matching}:}~
Given any maximum weight matching $M$ for $G'$,  a saturated
maximum weight matching $M'$ for $G'$ with the same weight
as $M$ can be constructed.
\begin{proof}
Recall that for each agent~$x_i$, there are~$|K_i|$ copies denoted
by~$x_i^t$,~$t\in K_i$ in~$G'$ by our construction. First, we will show by contradiction that
every type-3 resource is matched in~$M$.
Suppose a type-3 resource~$y^3_{i,p}$ is not matched. 
By our construction, this resource has edges only to
copies of nodes corresponding to $x_i$.
There are two cases.

\noindent
\underline{Case 1:}~
There is a node $x_i^{t'} \in X'$ that is not
matched in $M$. Here, we can add the edge 
$\{x_i^{t'}, y^3_{i,p}\}$ with weight $k+1 \geq 2$ to $M$;
this would contradict the fact that $M$ is a maximum
weight matching for $G'$.

\noindent
\underline{Case 2:}~
Every copy of $x_i$ is matched in $M$.
In this case, at least one copy of $x_i$, say $x_i^{t'}$,
is matched to a type-1 or type-2 resource node, and
weight of the corresponding edge is at most 1.
We can replace that edge in $M$ by the edge
$\{x_i^{t'}, y^3_{i,p}\}$ with weight $k+1 \geq 2$
to create a new matching whose weight is larger than
that of $M$. Again, this contradicts the maximality
of $M$. 

\smallskip
We thus conclude that every type-3 resource node 
is matched in $M$.

\smallskip
We now argue that $M$ can be modified to include
all the nodes of $X'$ without reducing
the total weight. 
Consider any agent $x_i$.
As argued above, $|K_i|-\rho_i$ copies of $x_i$
are matched to type-3 resources in $M$.
Now consider the remaining~$\rho_i$ copies of~$x_i$.  Suppose a
copy~$x_i^{t'}$ is unmatched in~$M$. Recall that every type-2 resource is
adjacent to every~$x_i^t$ and no other~$x_{i'}^t$ where~$i'\ne i$. Also,
there are~$\rho_i$ type-2 resources. This implies there exists at least one
unmatched type-2 resource to which~$x_i^{t'}$ can be matched. Also, we note
that such a scenario occurs only if the weight corresponding to edge
of~$x_i^{t'}$ and any unmatched type-2 resource is~$0$. (Otherwise, by
adding this edge, one can increase the weight of the matching contradicting
the fact that~$M$ is a maximum weighted matching.)

\smallskip
Therefore, if~$M$ is not saturated, we can construct a saturated
matching~$M'$ from~$M$ as follows. We include all the edges of~$M$ in~$M'$.
Then, we match each unmatched copy~$x_i^{t'}$ to an available type-2
resource. Thus, all nodes of~$X'$ are matched in~$M'$.
\end{proof}


\subsection{Statement and Proof of Lemma~\ref{lem:benefit-to-matching}}

\medskip

\noindent
\textbf{Statement of Lemma~\ref{lem:benefit-to-matching}:}~
Suppose there is a solution to the \maxtbmrm{} problem instance with
benefit $Q$. Then, there  a saturated matching $M$ for $G'$ 
with weight $Q + \lambda$ can be constructed.
\begin{proof}
Let~$\mbsoln$ be the solution to the \maxtbmrm{} problem. We will
construct a saturated matching~$M$ with weight $Q+\lambda$ as follows. Consider any
agent~$x_i$. If~$x_i$ is matched to~$y_j$ in round~$t$ in~$\mbsoln$,
then,~$\{x_i^t,y_j^t\}$ is added to~$M$. There are~$\gamma_i$ such edges.
Of the remaining~$|K_i|-\gamma_i$ copies of~$x_i$,~$|K_i|-\rho_i$ copies
are matched to type-3 resources. The remaining~$\rho_i-\gamma_i$ copies are
matched to type-2 resources in the following manner.
Let~$\overline{X}_i=\{x_i^{t_1},x_i^{t_2},\ldots,x_i^{t_{\rho_i-\gamma_i}}\}$
be the remaining copies ordered in an arbitrary manner. We will add the
edges~$\{x_i^{t_\ell},y^3_{i,\rho_i-\ell+1}\}$ to~$M$
for~$\ell=1,\ldots,\rho_i-\gamma_i$. Note that every~$x_i^t$,~$t\in K_i$ is
matched, and therefore, every node in~$X'$ is matched.

\smallskip
Now, we compute the total weight of~$M$. 
Again, consider any agent $x_i$. Each edge corresponding to a type-3
node contributes weight~$k$, and there are~$|K_i|-\rho_i$ such edges for~$x_i$. Hence, the total contribution from type-3 resources
is~$k(|K_i|-\rho_i)$. There are~$\gamma_i$ edges corresponding to type-1
resources. Since each such edge has weight~$1$, it contributes~$\gamma_i$
to the sum. Finally, the edges corresponding to type-2 resources
contribute~$\sum_{\ell=1}^{\rho_i-\gamma_i}w(x_i^{t_\ell},y^2_{i,\rho_i-\ell+1})
= \sum_{\ell=1}^{\rho_i-\gamma_i}\big(1-\increward_i(\rho_i-\ell+1)\big)$.
Combining the terms corresponding to type-1 and type-2 resources, 
the total weight of the edges in $M$ due to agent $x_i$ is
\(
\gamma_i +
\sum_{\ell=1}^{\rho_i-\gamma_i}\big(1-\increward_i(\rho_i-\ell+1)\big) = 
\sum_{\ell=1}^{\gamma_i}1 +
\sum_{\ell=\gamma_i}^{\rho_i}\big(1-\increward_i(\ell)\big) 
= \sum_{\ell=1}^{\rho_i}\big(1-\increward_i(\ell)\big) +
\sum_{\ell=1}^{\gamma_i}\increward_i(\ell)
\)\,.
Combining this over all the agents, the total
weight of the matching $M$ is given by~$\sum_{i=1}^nk(|K_i|-\rho_i) +
\sum_{i=1}^n\sum_{\ell=1}^{\rho_i}\big(1-\increward_i(\ell)\big) +
\sum_{i=1}^n\sum_{\ell=1}^{\gamma_i}\increward_i(\ell)=\lambda + Q$,
as stated in the lemma.
\end{proof}

\medskip


\subsection{Statement and Proof of Lemma~\ref{lem:opt-benefit-matching}}

\medskip

\noindent
\textbf{Statement of Lemma~\ref{lem:opt-benefit-matching}:}
There is an optimal solution to the MaxTB-MRM problem instance with benefit $Q^*$
if and only if there is a maximum weight saturated matching $M^*$
for $G'$ with weight $Q^* + \lambda$.

\begin{proof} Suppose there is an optimal solution to the \maxtbmrm{}
instance with benefit $Q^*$.  By
Lemma~\ref{lem:benefit-to-matching}, there is a saturated matching $M^*$
for $G'$ with weight $Q^* + \lambda$.  Now, if $G'$ has a saturated
matching with weight $W'$ \emph{larger than} $Q^* + \lambda$, then by
Lemma~\ref{lem:matching-to-benefit}, we would have a solution to the
\maxtbmrm{} problem with benefit $W'-\lambda > Q^*$.  This contradicts the
assumption that $Q^*$ is the maximum benefit for the \maxtbmrm{} instance.
In other words, $M^*$ is a maximum weight saturated matching for $G'$.  The
other direction can be proven in similar way.
\end{proof}

\medskip


\subsection{Statement and Proof of Proposition~\ref{pro:time-algmaxtbmrm}}

\medskip
\textbf{Statement of Proposition~\ref{pro:time-algmaxtbmrm}:}~
Algorithm \algmaxtbmrm{} runs in time 
$O(k^{3/2} (n+|E|)\sqrt{n+m}))$, where $n$ is the 
number of agents, $m$ is the number of resources,
$k$ is the number of rounds and $E$ is the number of
edges in the compatibility graph $G(X, Y, E)$.

\begin{proof} We first estimate the number of nodes and
edges in the bipartite graph $G'(X', Y', E')$
constructed in Step~1 of \algmaxtbmrm{}.
For each agent, there are at most $k$ nodes in $X'$.
Thus, $|X'| \leq kn$.
For each resource, there are $k$ nodes corresponding to
type-1 resources.
Thus, the total number of type-1 resource nodes is $km$.
For each agent, the total number of type-2 and type-3 resource nodes is at most $2k$. Therefore, the total number of type-2 and type-3
resource nodes over all the agents is $2kn$.
Thus, $|Y'| \leq km+ 2kn$.
Thus, the number of nodes in $G'$ is at most $km+3kn$.
We now estimate $|E'|$. 
For each edge $e \in E$, there are $k$ edges in $E'$
corresponding to edges incident on type-1 resources.
Thus, the number of edges of $G'$ incident on type-1 resource nodes
is at most $k|E|$.
As mentioned above, the total number of type-2 and type-3 nodes
for each agent is at most $2k$. Thus, the total number of edges
in $E'$ contributed by all agents to type-2 and type-3 resource nodes is at most $2kn$. Thus, $|E'| \leq k|E| + 2kn$.
It can be seen that the running time of the algorithm is dominated
by Step~2, which computes a maximum weighted matching in $G'$.
It is well known that for a bipartite graph with $p$ nodes and $q$
edges, a maximum weighted matching can be computed in
time $O(q\sqrt{p})$ \cite{CLRS-2009}.
Since $G'$ has at most $3kn+km$ nodes and at most
$k|E| + 2kn$ edges, the running time of \algmaxtbmrm{}
is $O(k^{3/2} (|E|+n)\sqrt{n+m}))$.
\end{proof}

\medskip

\subsection{Statement and Proof of Theorem~\ref{thm:rawlsian}}

\medskip

\noindent
\textbf{Statement of Theorem~\ref{thm:rawlsian}:}~
There exists a benefit function $\mu_i$ for each agent $x_i$ such that maximizing the total benefit under this function
maximizes the Rawlsian social welfare function,
i.e., it maximizes the minimum
satisfaction ratio over all agents. 
Further, this benefit function is valid 
and therefore, an optimal solution to the \maxtbmrm{}
problem under this benefit function can
be computed in polynomial time.

\smallskip

The benefit function mentioned in the above theorem was
defined in Section~\ref{sec:rmmrm_poly}.
For the reader's convenience, we have reproduced the
definition below.

\smallskip

\noindent
\begin{enumerate}[leftmargin=*]
\item Let~$F=\big\{k_1/k_2\mid k_1,k_2\in\{1,\ldots,k\} ~\mathrm{and}~ k_1 \leq k_2\big\}\cup\{0\}$.
\item Let~$\ord: F\rightarrow\{1,\ldots,|F|\}$ correspond to the index of
each element in~$F$ when sorted in descending order.
\item For each~$q\in F$, let $\xi(q)=(nk)^{\ord(q)}/(nk)^{|F|}$, where~$n$ is
the number of agents and~$k$ is the number of rounds. Note that
each~$\xi(q)\in(0,1]$.
\item The incremental benefit~$\delta_i(\ell)$ for an agent~$x_i$ for
the~$\ell$th matching is defined
as~$\delta_i(\ell)=\xi\big((\ell-1)/\rho_i\big)$.
Thus, the benefit function $\mu_i$ for agent $x_i$ is
given by $\mu_i(0) = 0$ and 
$\mu_i(\ell) = \mu_i(\ell-1) + \delta_i(\ell)$ for
$1 \leq \ell \leq \rho_i$.
\end{enumerate}
It can be seen that $\mu_i$ satisfies
properties P1, P2 and P4 of a valid benefit
function. We now show that it also satisfies P3,
the diminishing returns property.
\begin{lemma}
For each agent~$x_i$, the incremental benefit~$\delta_i(\ell)$ is monotone
non-increasing in~$\ell$. Therefore,~$\mu_i(\cdot)$ satisfies
diminishing returns property.
\end{lemma}
\begin{proof}
By definition, for~$\ell=1,\ldots,\rho_i-1$,~$\delta_i(\ell)/\delta_i(\ell+1)=
\xi\big((\ell-1)/\rho_i\big)/\xi\big(\ell/\rho_i\big)\ge nk > 1$, by noting
that~$\pi(\ell/\rho_i)-\pi((\ell+1)/\rho_i)\ge1$. Hence, the lemma holds.
\end{proof}
To complete the proof of Theorem~\ref{thm:rawlsian},
we need to show that any solution that maximizes the benefit function defined above also maximizes
the minimum satisfaction ratio.
We start with some definitions and a lemma.

\smallskip

For an agent~$x_i$ and~$q\in F$, let~$\Fg_i(q)=\{\ell\mid 0\le\ell\le\rho_i
\text{ and } \ell/\rho_i>q\}$. We will now show that the incremental
benefit~$\delta_i(\ell)$ obtained by agent~$x_i$ for matching~$\ell$ is
greater than the sum of all incremental benefits~$\delta_{i'}(\ell')$, for
all agents~$x_{i'}$, $i'\ne i$ and for all~$\ell'$
satisfying~$\ell'/\rho_{i'}>\ell/\rho_i$.
\begin{lemma}\label{lem:cost}
For any agent~$x_i$ and non-negative
integer~$\ell\le\rho_i$, we have ~$\delta_i(\ell)>\sum_{i'=i}\sum_{\ell'\in
\Fg_{i'}(\ell/\rho_i)}\delta_{i'}(\ell')$.
\end{lemma}
\begin{proof}
For any~$\ell'\in \Fg_{i'}(\ell/\rho_i)$, note that
$\ell/\rho_i<\ell'/\rho_{i'}$.
Hence,~$\pi(\ell/\rho_i)-\pi(\ell'/\rho_{i'})\ge1$. This implies
that~$\delta_i(\ell)/\delta_{i'}(\ell') =
\xi\big((\ell-1)/\rho_i\big)/\xi\big((\ell'-1)/\rho_{i'}\big)\ge nk$.
Since~$\rho_{i'}\le k$, the number of 
$\delta_{i'}(\ell')$ terms
per~$i'$ is at most $k$.
Since there are at most~$n-1$ agents~$x_{i'}$,~$i'\ne i$,~ it
follows that there are at most~$(n-1)k$ terms~$\delta_{i'}(\ell')$ in total. Therefore,~$\sum_{i'=i}\sum_{\ell'\in
\Fg_{i'}(\ell/\rho_i)}\frac{\delta_{i'}(\ell')}{\delta_{i}(\ell)}<1$.
\end{proof}

\smallskip
We are now ready to complete the proof of Theorem~\ref{thm:rawlsian}.

\medskip


\noindent
\textbf{Proof of Theorem~\ref{thm:rawlsian}.}
The proof is by contradiction. Suppose~$\mbsoln^*$ is an optimal solution
given the benefit function defined above. We will show that if there exists
a solution~$\mbsoln$ with minimum satisfaction ratio greater than that
of~$\mbsoln^*$, then the total
benefit~$\totreward(\mbsoln)>\totreward(\mbsoln^*)$, contradicting the fact
that~$\mbsoln^*$ is an optimal solution. Let~$q(\mbsoln^*)$
and~$q(\mbsoln)$ denote the minimum satisfaction ratio in~$\mbsoln^*$
and~$\mbsoln$.
For a given solution $\mbsoln$, we use $\gamma_i(\mbsoln)$
to denote the number of rounds assigned to agent $x_i$,
$1 \leq i \leq n$.

We recall that the benefit function for~$\mbsoln$ can be written
as~$\totreward(\mbsoln)=\sum_{i=1}^n\sum_{\ell=1}^{\gamma_i(\mbsoln)}
\delta_i(\ell)$. For each~$i$, the~$\delta_i(\ell)$ terms can be
partitioned into three blocks as follows.

\begin{description}
\item{(i)} $D_1(\mbsoln)=\{(i,\ell)\mid \forall
i,~\ell/\rho_i\le q(\mbsoln^*)\}$, 
\item{(ii)} $D_2(\mbsoln)=\{(i,\ell)\mid
\forall i,~q(\mbsoln^*)<\ell/\rho_i\le q(\mbsoln)\}$, and
\item{(ii)} $D_3(\mbsoln)=\{(i,\ell)\mid \forall i,~\ell>q(\mbsoln)\}$. 
\end{description}

\noindent
We can
partition the terms corresponding to~$\totreward(\mbsoln^*)$ in the same
way, and these blocks are denoted by~$D_1(\mbsoln^*)$, $D_2(\mbsoln^*)$ and $D_3(\mbsoln^*)$.
Since
every~$x_i$ has a satisfaction ratio of at least~$q(\mbsoln^*)$ in
both~$\mbsoln^*$ and~$\mbsoln$, all the terms with~$\ell/\rho_i\le
q(\mbsoln^*)$ will be present. Therefore,~$D_1(\mbsoln^*)=D_1(\mbsoln)$.

\smallskip

In~$\mbsoln^*$, let~$x_{i'}$ be an agent for which the satisfaction ratio
is~$q(\mbsoln^*)$. There are no~$(i',\ell)$ terms in~$D_2(\mbsoln^*)$ while
for every agent~$x_i$, all~$(i,\ell)$ terms
satisfying~$q(\mbsoln^*)<\delta_i(\ell)\le q(\mbsoln)$ are present
in~$D_2(\mbsoln)$. Therefore,~$D_2(\mbsoln^*)\subset D_2(\mbsoln)$. Now,
\begin{align*}
\totreward(\mbsoln)-\totreward(\mbsoln^*) = & 
\sum_{(i,\ell)\in D_2(\mbsoln)}\delta_i(\ell) +
\sum_{(i,\ell)\in D_3(\mbsoln)}\delta_i(\ell) -
\sum_{(i,\ell)\in D_2(\mbsoln^*)}\delta_i(\ell) -
\sum_{(i,\ell)\in D_3(\mbsoln^*)}\delta_i(\ell) \\
& >
\sum_{(i,\ell)\in D_2(\mbsoln)}\delta_i(\ell) -
\sum_{(i,\ell)\in D_2(\mbsoln^*)}\delta_i(\ell) -
\sum_{(i,\ell)\in D_3(\mbsoln^*)}\delta_i(\ell) \\
\end{align*}
Now, let~$(i',\ell')=\argmin_{(i,\ell)\in D_2(\mbsoln)\setminus
D_2(\mbsoln^*)} \ell/\rho_i$. Also, recalling that~$D_2(\mbsoln^*)\subset
D_2(\mbsoln)$,
\begin{align*}
\totreward(\mbsoln)-\totreward(\mbsoln^*) & >
\sum_{(i,\ell)\in D_2(\mbsoln)}\delta_i(\ell) -
\sum_{(i,\ell)\in D_2(\mbsoln^*)}\delta_i(\ell) -
\sum_{(i,\ell)\in D_3(\mbsoln^*)}\delta_i(\ell) \\
& > \delta_{i'}(\ell') -
\sum_{(i,\ell)\in D_3(\mbsoln^*)}\delta_i(\ell) > 0\\
\end{align*}
where the inequality follows from Lemma~\ref{lem:cost}.
Thus, $\totreward(\mbsoln)>\totreward(\mbsoln^*)$ and this
contradicts the optimality of $\mbsoln^*$.
The theorem follows.
\qed

\medskip

\subsection{Proof of the Complexity of \maxsamrm{}}

\medskip

\noindent
\textbf{Statement of Theorem~\ref{thm:max_sat_agents_npc}:}
The \maxsamrm{}{} problem is \cnp-hard even when
the number of rounds ($k$) is 3.

\medskip

\noindent
\textbf{Proof:}~ 
To prove \cnp-hardness, we use a reduction from the Minimum Vertex Cover
Problem for Cubic graphs, which we denote as \mvcc{}.
A definition of this problem is as follows: given
an undirected graph $H(V_H, E_H)$ where the degree of each node is 3
and an integer $h \leq |V_H|$, is there a vertex cover of size at most $h$
for $H$ (i.e., a subset $V'_H \subseteq V_H$
such that $|V'_H| \leq h$, and for each edge $\{x,y\} \in E_H$, at least
one of $x$ and $y$ is in $V'_H$)?
It is known that \mvcc{} is \cnp-complete \cite{GareyJohnson79}.

\medskip
Given an instance $I$ of \mvcc{} consisting of graph $H(V_H, E_H)$
and integer $h \leq |V_H|$, we produce an instance $I'$ of \msmrm{}
as follows.

\begin{enumerate}
\item For each node $v_i \in V_H$, we create an agent $x_i$,
called a \textbf{special agent}.
For each edge $e_j \in E_H$, we create 
two agents $z_{j,1}$ and $z_{j,2}$, called \textbf{simple agents}.
Thus, the set $X$ of agents consists of $|V_H|$ special agents
and $2|E_H|$ simple agents for a total of $|V_H| + 2|E_H|$ agents.
\item For each edge $e_j \in E_H$, we create
a resource $y_j$.
Thus, the set $Y$ of resources has size $|E_H|$.

\item The edge set $E$ of the compatibility graph $G(X, Y, E)$
is constructed as follows.
For $a \in \{1,2\}$, each simple agent $z_{j,a}$ has an edge in $E$ only to
its corresponding resource $y_j$.
Further, for each edge $e_j = \{v_a, v_b\} \in E_H$,
$E$ has the two edges $\{x_a, y_j\}$ and $\{x_b, y_j\}$.
Since the degree of each node $v_i \in V_H$ is 3,
each special agent has exactly three compatible resources.
The edge set $E$ ensures that each resource is compatible
with exactly two simple agents and two special agents.
It follows that $|E| = 4|Y| = 4|E_H|$.

\item For each agent, the allowed set of rounds is $\{1, 2, 3\}$.

\item The requirement for each simple agent is 1 and
that for each special agent is 3.

\item The number $k$ of rounds is set to 3.

\item The parameter $Q$, the number of agents to
be satisfied, is given by $Q ~=~ (2|E_H| + |V_H| -h)$,
where $h$ is the bound on the size of a vertex cover for $H$.
\end{enumerate}
It can be seen that the above construction can be carried out
in polynomial time.
We now show that the resulting instance $I'$ of \msmrm{}
has a solution iff the \mvcc{} instance $I$ has a solution.

\smallskip
Suppose the \mvcc{} instance $I$ has a solution.
Let $V'_H$ denote a vertex cover of size $h$ for $H$.
We construct a solution for the instance $I'$ of \msmrm{}
consisting of three matchings $M_1$, $M_2$ and $M_3$
as follows.
Initially, these matchings are empty.
Let $X' \subseteq X$ consist of all the simple agents and
those special agents corresponding to the nodes of $H$ in $V_H - V'_H$.
(Thus, the special agents in $X'$ correspond to the nodes of $H$
which are \emph{not} in $V'_H$.)
First, consider each special agent $x_i \in X'$.
Let the three resources that are compatible with $x_i$ be denoted
by $y_a$, $y_b$ and $y_c$ respectively.
We add the edges $\{x_i, y_a\}$, $\{x_i, y_b\}$, and
$\{x_i, y_c\}$ to $M_1$, $M_2$ and $M_3$ respectively.
Since $V'_H$ is a vertex cover, each resource is used in
one of the three rounds by a special agent.
Thus, each of the two simple agents $z_{j,1}$ and $z_{j,2}$
for each resource $y_j$ can be matched in the two remaining
rounds. (For example, if agent $x_i$ is matched to $y_j$ in $M_1$,
$z_{j,1}$ and $z_{j,2}$ can be matched to $y_j$ in $M_2$ and $M_3$
respectively.)
Thus, in this solution, the only unsatisfied agents are the
$h$ special agents corresponding to the nodes in $V'_H$.
In other words, at least $2|E_H| + |V_H| -h$ agents are satisfied.

\smallskip

For the converse, suppose there is a solution to the \msmrm{}
instance $I'$ that satisfies at least
$2|E_H| + |V_H| -h$ agents.
Let $X'$ denote the set of satisfied agents in this solution.
We have the following claim.

\smallskip

\noindent
\textbf{Claim 1:}~
$X'$ can be modified without changing its
size so that it includes \emph{all} the simple agents.

\smallskip

\noindent
\textbf{Proof of Claim 1:}~
Suppose a simple agent $z_{j,a}$ for some $a \in \{1,2\}$
is not in $X'$ (i.e., $z_{j,a}$ is not a satisfied agent).
By our construction, $z_{j,a}$ is only compatible with resource $y_j$.
Since $y_j$ is compatible with two simple agents and two special agents,
the current solution must assign at least one special agent,
say $x_w$, to $y_j$ in one of the matchings.
In that matching, we can replace $x_w$ by $z_{j,a}$ so that $z_{j,a}$
becomes a satisfied agent and $x_w$ becomes an unsatisfied agent.
This change corresponds to adding $z_{j,a}$ to $X'$ and
deleting $x_w$ from $X'$. Thus, this change does \emph{not}
change $|X'|$. By repeating this process, we ensure that $X'$ contains
only special agents and $|X'| ~\geq~ 2|E_H| + |V_H| -h$.

\smallskip

We now continue the main proof.
From the above discussion,  all the unsatisfied agents, that is, the agents
in $X - X'$, are special agents and their number is at most $h$.
Let $V'_H$ denote the nodes of $H$ that correspond to the
agents in $X-X'$.
To see that $V'_H$ forms a vertex cover for $H$, consider
any edge $e_j = \{v_a, v_b\}$ of $H$.
We show that $V'_H$ contains at least one of $v_a$ and $v_b$.
There are four agents compatible with the resource $y_j$
that corresponds to $e_j$.
Two of these are simple agents and the other two are special agents.
However, the number of rounds is 3.
Therefore, there is at least one unsatisfied agent that is
compatible with $y_j$.
By Claim~1, such an unsatisfied agent is a special agent.
The only special agents that are compatible with $y_j$ are
$x_a$ and $x_b$. Thus, at least of one these agents
appears in $X-X'$.
By our construction of $V'_H$, at least one of $v_a$ and $v_b$
appears in $V'_H$. Further, $|V'| \leq h$.
Thus, $V'_H$ is a solution to the \mvcc{} instance $I'$.
This completes our proof of Theorem~\ref{thm:max_sat_agents_npc}. \hfill$\Box$

\bigskip

\section{Additional Material for Section~\ref{sec:algmsagmrm}}
\label{sup:sec:algmsagmrm}

\medskip

\subsection{Proof of the Complexity of \agmrm{}}

\medskip


\paragraph{Statement of Theorem~\ref{thm:hardness}.}
The \agmrm{} problem is NP-hard when there is just one
agent~$x_i$ for which~$\rho_i > 1$.

\medskip

\noindent
\textbf{Proof.}~ We use a reduction from the
Minimum Set Cover (MSC) problem which is known to
be \cnp-complete~\cite{GareyJohnson79}.
Consider an instance of MSC
with~$U=\{z_1,z_2,\ldots, z_t\}$,~$\setc=\{A_1,A_2,\ldots, A_q\}$, and an
integer~$\alpha\le|\setc|$. We will construct an instance of \agmrm{}{}
as follows. Let~$t =|U|$. 
Let degree of an element~$z\in U$,
denoted by~$\deg(z)$, be the number of sets
in~$\setc$ which contain $z$. 
For each~$z_j\in U$, we
create~$a=\deg(z_j)$ resources~$\rset_j=\{y_j^h\mid z_j\in A_h\}$. Thus,
the resource set is~$\rset = \bigcup_{z_j\in U}\rset_j$. For each~$\rset_j$,
we create a set of agents~$\agset_j = \{x_j^1,x_j^2,\ldots\}$,
where~$|\agset_j| = t\cdot\deg(z_j)-1$. For each~$z_j\in U$, we have an
edge for every~$x\in \agset_j$ and~$y\in \rset_j$. Finally, we add a
special agent~$x^*$ with edges to all resources. The set of agents
is~$\agset=\bigcup_{z_j\in U}\agset_j\cup\{x^*\}$. For each~$x\ne x^*$,
there are no labels on the edges that are incident on $x$. For each
edge~$\{x^*,y_j^h\}$, we assign a single label~$c_h$. Each label has
cost~$1$. This completes the construction of~$G_R$. For~$x^*$, the budget
$\beta^*=\alpha$, which is the budget for the MSC instance. For the
rest of the agents, the budget is~$0$. The number of repetitions
for~$x^*$,~$\rho^*=t$ and for the rest of the vertices~$x_i\ne
x^*$,~$\rho_i=1$. This completes the construction of the \agmrm{} instance,
and it can be seen that the construction can be done in polynomial time.
Now, we  show that a solution exists for the MSC
instance iff  a solution exists for the resulting \agmrm{} instance.

\smallskip
Suppose that there exists a solution to the \agmrm{} instance. By our
construction, since~$t\cdot\deg(z_j)-1$ agents in~$\agset_j$ are all
adjacent to the same set of~$\deg(z_j)$ resources, it follows that~$x^*$
can be matched to at most one resource in~$\rset_j$ in the~$k$ rounds.
The solution to the~$\agmrm{}$ instance matches~$x^*$ to~$t$
resources, one in each matching. Thus, $x^*$ is adjacent to
some~$y_j^h\in\rset_j$ for all~$j$. 
This means that the
collection~$\setc'=\{A_h\mid y_j^h\in\rset_j \text { and } x^*\text{
compatible with } y_j^h\}$ covers all the elements in $U$. 
Further, since at most~$\alpha$ labels were removed, it follows that~$|\setc'|=\alpha$. In other words, we have a solution to the
MSC instance.

\smallskip
Let~$\setc'$ be a solution to the MSC instance. We can construct a
solution with~$k$ matchings for the \agmrm{} instance as follows. First,
we remove restrictions~$c_h$ $\forall A_h\in \setc'$. Since~$\setc'$
covers~$U$, for every~$\rset_j$, there exists a resource~$y_j^h$ such
that~$A_h\in \setc'$ and therefore,~$x^*$ is compatible with~$y_j^h$. We
choose exactly one such resource from each~$\rset_j$. Let this chosen set
of resource be denoted by~$\rset'$.  Since~$t=|U|$ matchings are required
and~$|\rset'|=t$, we can choose a distinct~$y\in \rset'$ to match~$x^*$.
This means that each resource in~$Y_j\setminus\rset'$ is available to be
matched to an agent in~$t$ rounds while the resource~$y\in \rset_j\cap\rset'$
can be matched in~$t-1$ rounds. Since there are exactly~$t\cdot\deg(z_j)-1$
resources in~$\agset_j$, each of them can be matched in one of the~$t$
rounds. Hence, we have a solution to the \agmrm{} instance and this completes our proof 
of \cnp-hardness. \qed

\medskip


\subsection{An ILP Formulation for \agmaxsamrm{}}

\medskip

\begin{table}[htbp]
\begin{center}
\begin{tabular}{|c|p{4.5in}|}\hline
\textbf{Symbol} & \textbf{Explanation}\\ \hline\hline
\agset{} & Set of $n$ agents; \agset{} = $\{x_1, x_2, \ldots, x_n\}$ \\ \hline
\rset{} & Set of $m$ resources; \rset{} = $\{y_1, y_2, \ldots, y_m\}$ \\ \hline
$\calc_i$ & Set of labels $\{c_i^1, c_i^2, \ldots, c_i^{\ell_i}\}$
            of associated with agent $x_i$, $1 \leq i \leq n$.
            (Thus, $\ell_i = |\calc_i|$.)
            \\ \hline
$\psi_i^t$ & Cost of removing the label $c_i^t$, $1 \leq t \leq \ell_i$ and
               $1 \leq i \leq n$.  \\ \hline
$\beta_i$ & Cost budget for agent $x_i$, $1 \leq i \leq n$.
            (This is the total cost agent $x_i$ may spend
             on removing labels.)
            \\ \hline
$\rho_i$  & The desired number of rounds requested by an 
            agent. \\ \hline
$\Gamma_e$ & Set of labels appearing on edge $e$. (If an edge $e$ joins
             agent $x_i$ to resource $y_j$, then $\Gamma_e \subseteq \calc_i$.)
             Edge $e$ can be added to the compatibility graph only
             after \emph{all} the labels in $\Gamma_e$ are removed.
            \\ \hline
$k$ & Upper bound on the number of rounds of matching \\ \hline
\end{tabular}
\end{center}
\caption{Notation Used in the Definition of \agmaxsamrm{}{}{}{}
and its ILP Formulation}
\label{tab:notation}
\end{table}

We now provide an ILP formulation for the \agmaxsamrm{} problem.
Recall that the goal of this problem is to generate advice to agents
(i.e., relaxation of restrictions) subject to budget constraints
so that there is $k$-round matching that satisfies the maximum
number of agents.
For an agent $x_i$ and resource $y_j$, 
if the edge $e
= \{x_i, y_j\}$ does \emph{not} appear in the restricted graph, it is
assumed the cost of removing the labels associated with that edge is larger than agent $x_i$'s budget.  
This will ensure that the edge won't appear in the compatibility graph.
The reader may want to consult
Table~\ref{tab:notation} for the notation used in this
ILP formulation.

\medskip

\noindent
\textbf{\ilpagmaxsamrm{}:}

\medskip
\noindent
\textbf{Note:}~ This ILP is also referred to as
\textsc{ILP-MaxSA-AG-MRM}.

\medskip

\noindent
\underline{\textsf{Variables:}} 
Throughout this discussion, unless specified otherwise, 
we use $\forall i$, $\forall j$ and $\forall r$~ to
mean ~$1\le i\le n$, $1\le j\le
m$, and $1\le r\le k$ respectively.

\begin{itemize}[leftmargin=*,noitemsep,topsep=0pt]

\item For each label $c_i^t$, a \{0,1\}-valued variable $w_i^t$,
$1 \leq t \leq \ell_i$ such that
$w_i^t$ is 1 iff label $c_i^t$ 
is removed.

\item For each agent $x_i$ and resource $y_j$,
a \{0,1\}-valued variable~$z_{ij}$ such that
$z_{ij}$ is 1 iff the edge $\{x_i, y_j\}$
appears in the compatibility graph (after removing labels).

\item 
For each agent $x_i$, 
and each resource $y_j \in \rset$,
we have $k$ \{0,1\}-variables
denoted by $a_{ijr}$.
The interpretation is that  $a_{ijr}$ is 1 iff edge $\{x_i, y_j\}$
is in the compatibility graph and agent $x_i$ is matched
to resource $y_j$ in round $r$.
\item
For each agent $x_i$, we have an integer
variable $\eta_i$
that gives the number of rounds in which $x_i$ is matched.
\item
For each agent $x_i$, we have a \{0,1\}-variable $s_i$
that is 1 iff agent $x_i$'s requirement
(i.e., $\eta_i \geq \rho_i$) is satisfied.
\end{itemize}


\medskip

\noindent
\underline{\textsf{Objective:}}~
Maximize $\sum_{i=1}^{n} s_i$.

\medskip

\noindent
\underline{\textsf{Constraints:}} 
\begin{itemize} 
\item For each agent $x_i$, the cost of removing the labels from $\calc_i$
must satisfy the budget constraint:
\(\forall i~
\sum_{t=1}^{\ell_i} \psi_i^t\,w_i^t ~\leq~ \beta_i\,.
\)

\item The edge $e = \{x_i, y_j\}$ appears in the graph only if
\emph{all} the labels in $\Gamma_e$ are removed.
\begin{align*}
z_{ij} \leq~& w_i^t ~~ \mathrm{for~each~} c_i^t \in \Gamma_e \\
z_{ij} \geq~& \bigg[\sum_{c_i^t \in \Gamma_e} w_i^t\bigg] - (|\Gamma_e| -1)
\end{align*}
The first constraint ensures that if any of the labels on the edge
$\{x_i, y_j\}$ is not removed, edge $\{x_i, y_j\}$ won't be added
to the compatibility graph.
The second constraint ensures that if all of the labels on the edge
$\{x_i, y_j\}$ are removed, edge $\{x_i, y_j\}$ \emph{must be}  added
to the compatibility graph.
These constraints are for $\forall i$ and $\forall j$.

\item If an edge $\{x_i, y_j\}$ does not have any labels on it,
then it is included in the compatibility graph:
\(
z_{ij} ~=~1 ~~ \mathrm{for~each~} e =\{x_i,y_j\} ~\mathrm{with}~ \Gamma_e = \emptyset{}.
\) This constraint is for $\forall i$ and $\forall j$.

\item If edge $\{x_i, y_j\}$ is not in the compatibility graph, 
each variable $a_{ijr}$ 
must be set to 0: 
\(
    a_{ijr} ~\leq~ z_{ij},~ \forall i,~ \forall j,~ \forall k.
\)

\item
For each agent $x_i$, the number of rounds in which $x_i$
is matched should be equal to $\eta_i$:
\(
    \eta_i = \sum_{r=1}^{k}\, \sum_{j=1}^{m} a_{ijr},~ \forall i~
\)


\item Each agent $x_i$ is matched to a resource \emph{at most once}
in each round:
\(
    \sum_{j = 1}^m\: a_{ijr} ~\leq~ 1,~ \forall i,~ \forall r.
\)

\item For each resource $y_j$, 
in each round, $y_j$ must be matched to an agent
\emph{at most} once: 
\(
    \sum_{i = 1}^{n} a_{ijr} ~\leq~ 1,~ \forall j,~ \forall r. 
           \)

\item The rounds in which an agent $x_i$ is matched
must be in $K_i$: 
\(
    a_{ijr} ~=~ 0, ~~ r \not\in K_i,~ \forall i,~ \forall j.
 \)

\item
Variable $s_i$ should be 1
iff agent $x_i$ was matched
in at least $\rho_i$ rounds. The following constraints,
which must hold $\forall i$, ensure this
condition.
\begin{align}\label{eqn:s-i-ineq}
\eta_i - \rho_i + 1 \le k\,s_i \le \eta_i - \rho_i + k
\end{align}

\item $w_i^t \in \{0,1\}$, $z_{ij} \in \{0,1\}$, $a_{ijr} \in \{0,1\}$, 
$\eta_i$ is a non-negative integer, and $s_i \in \{0,1\}$.
\end{itemize}

\medskip

\smallskip

\noindent
\underline{\textsf{Obtaining a Solution from the ILP:}}~ 
For each variable $w_i^t$ set to 1, the corresponding
label $c_i^t$ is removed.
These variables can be used to get the set of all labels
each agent must remove; this is the advice that is given to
each agent.
For each variable $a_{ijr}$ which is set to 1 in the
solution, we match agent $x_i$ with resource $y_j$ in round $r$.

\bigskip

\section{Additional Material for Section~\ref{sec:experiments}}
\label{sup:sec:experiments}

\medskip

\subsection{Additional Classroom Example with Restrictions}
\label{app_sec:addl_example}

\begin{figure}[hbt]
\rule{\columnwidth}{0.01in}
\centering
\includegraphics[width=.5\columnwidth]{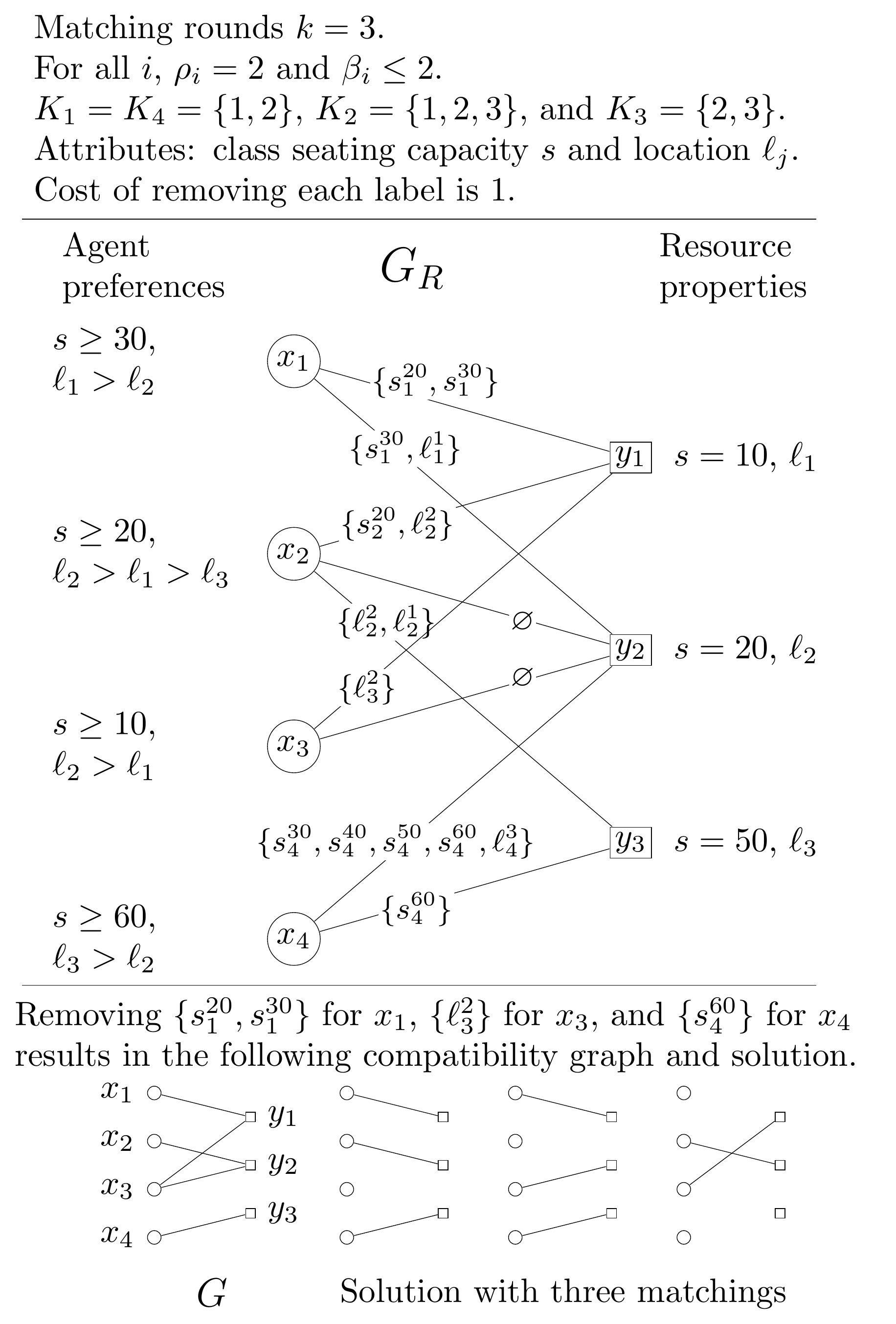}
\caption{An example of the $\agmrm$ problem.\label{fig:example2}}
\rule{\columnwidth}{0.01in}
\end{figure}

\medskip

\paragraph{A course-classroom example.} This example is motivated by the classroom
assignment application considered in this work. See
Figure~\ref{fig:example2}. Suppose there are agents  (cohorts) who need to
be assigned to classrooms on three days ($k=3$) at some specific time such
that each agent is assigned a room at least twice ($\forall i,\,\rho_i=2$).
There are three classrooms. Each agent has preferences regarding the
minimum seating capacity of the room, its location, and the rounds in which
it would like to be matched. For example, agent~$x_1$ requires the room to
have a  capacity (i.e., space) of~$\ge30$,  prefers a room in
location~$\ell_1$ over one in~$\ell_2$, and does not want a room
in~$\ell_3$ (a hard constraint). Also, it wants to be matched in rounds~$1$
and~$2$.  Resource~$y_1$ has capacity~$10$ and is situated in
location~$\ell_1$.  The restrictions graph based on agent preferences and
resource properties is constructed as follows. To minimize the number of
labels, we will assume that each agent can relax its capacity constraint
only in steps of~$10$.  For~$x_1$ to be compatible with~$y_1$, it must
relax its capacity requirement from~$30$ to~$10$.  To represent this, we
introduce two labels~$s_1^{30}$ and~$s_1^{20}$ on the edge between~$x_1$
and~$y_1$.  Removing $s_1^{30}$ implies that~$x_1$ has relaxed its capacity
restriction from~$30$ to~$20$.  If a resource is situated in location
$\ell_j$, then labels corresponding to all locations above~$\ell_j$ in the
agent's preference order are assigned to the edge.  For example, the edge
between~$x_2$ and~$y_3$ has labels~$\ell_2^2$ and~$\ell_2^1$ as~$y_3$
belongs to region~$\ell_3$, which is its least preferred location.  The
restriction graph for this problem is shown in the middle panel of
Figure~\ref{fig:example2}.  A solution to $\agmrm$ in the bottom panels
of the figure.

\subsection{Generation of restrictions graphs}

\medskip

In the real-world datasets, the agent restrictions are mostly induced by
the properties of resources (room size, WiFi connectivity, location, etc.).
Three different types of restrictions are considered: binary attributes
such as yes/no and present/absent; ordinal attributes such as preferences;
and quantitative attributes. The example in Section~\ref{sec:definitions}
has binary attributes. We discretize quantitative attributes like capacity
by generating labels in steps of some fixed value. Suppose an agent~$x$
prefers a resource~$y$ of capacity~5, but~$y$ has capacity~3, then, we add
labels~$\ell_5$ and~$\ell_4$ to the restrictions set of edge~$\{x,y\}$. To relax the capacity constraint and allow an agent to be matched to a
resource of capacity~3, the two labels $\ell_3$ and $\ell_4$ must be
removed. For ordinal attributes we use the following method: suppose~$x$
has the following preference order~$a>b>c$ for some attribute values~$a$, $b$,
and $c$ and~$y$ has attribute~$b$, we add restrictions~$\ell_a$ to the
restrictions set of edge~$\{x,y\}$. Relaxing the restriction to
accommodate~$b$ corresponds to removing~$\ell_a$. 
For both
real-world datasets, each round corresponds to a day.  The cost is assigned
as follows: for binary attributes, the cost is~1, while for quantitative
and ordinal attributes, the cost on a label depends on the distance from
the preferred value: in the above example, the cost~$\ell_5$ is~1 while the
cost of~$\ell_4$ is~2. The same applies to ordinal attributes.

\medskip

\subsection{An Additional Figure for the Lab-Space dataset}
\medskip

As in Figure~\ref{fig:ls_total}(d), Figure~\ref{fig:ag2students} shows that there is a sharp increase in the number of agents satisfied for a
budget of~4. Figure~\ref{fig:ag2students} also implies that the performance of
\heuristicagmrm{} is close to that of optimum (given by \ilpagmaxsamrm{}). Finally,  when budget is high,  the fraction of agents satisfied in the latter is reaching $100\%$ while in the former it is at most ~$90\%$.

\begin{figure}
\centering
\includegraphics[width=.48\columnwidth]{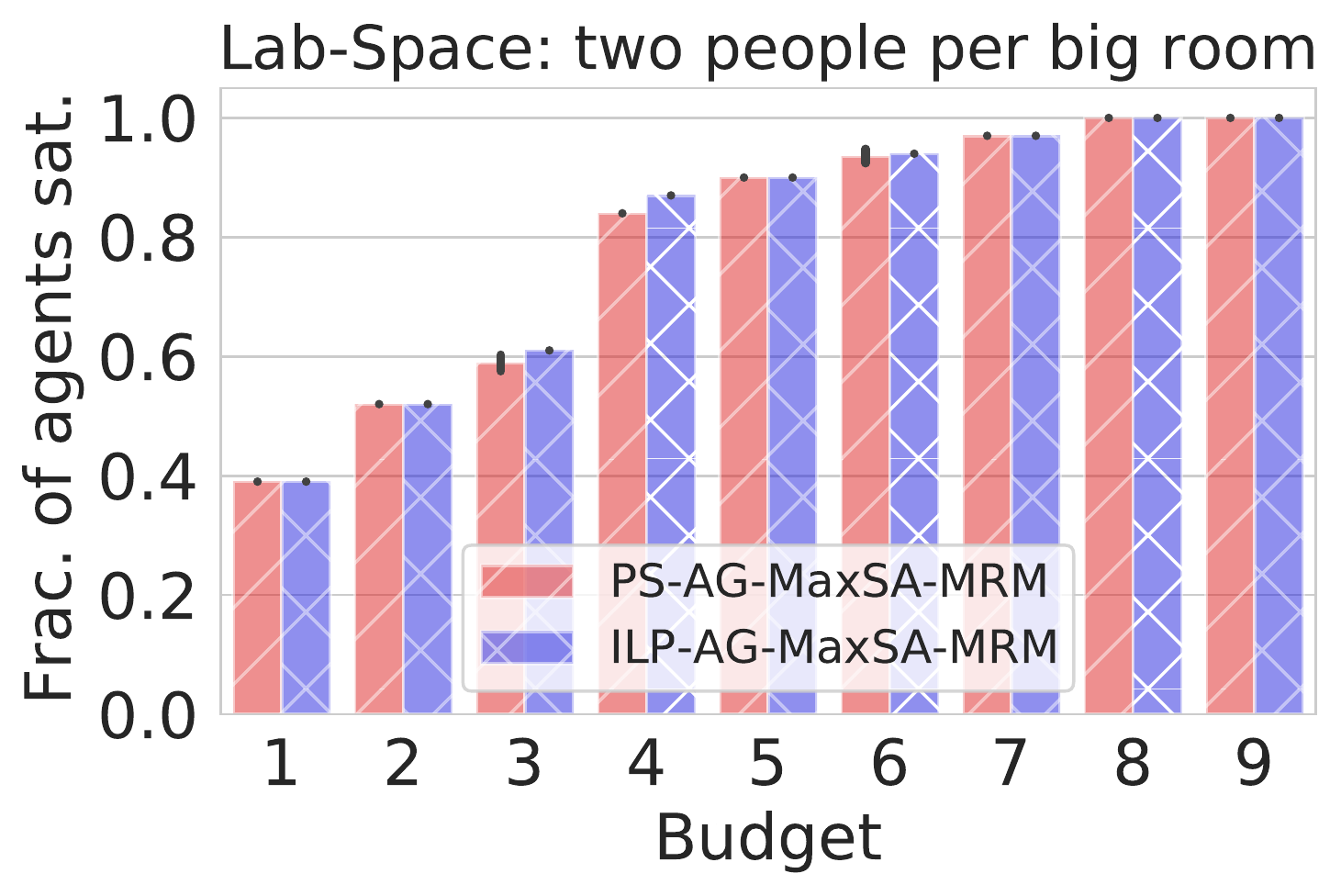}
\caption{\ls{} experiment results: Evaluation
of the advice generation algorithm for increasing budget when each big room may contain up to two students.
\label{fig:ag2students}
}
\end{figure}
\medskip

\subsection{An Additional Figure for the Course-Classroom dataset}
\medskip

In Figure~\ref{fig:prbresources} we
demonstrate how critical the current set of resources is for having a satisfactory solution for the ~\cc{} application.
We withheld only a portion of the resources (by sampling)
 to satisfy the agents' preferences. The motivation behind this is that some of the university's classrooms are to be dynamically allocated for some conferences and other events during the year, and the 
 \prin{} needs to decide in advance what capacity to keep free for such events.
The graph shows that with 80\% of the resources, it is possible to achieve an average assignment ratio of about ~$0.8$.

\begin{figure}
\centering
\includegraphics[width=.48\columnwidth]{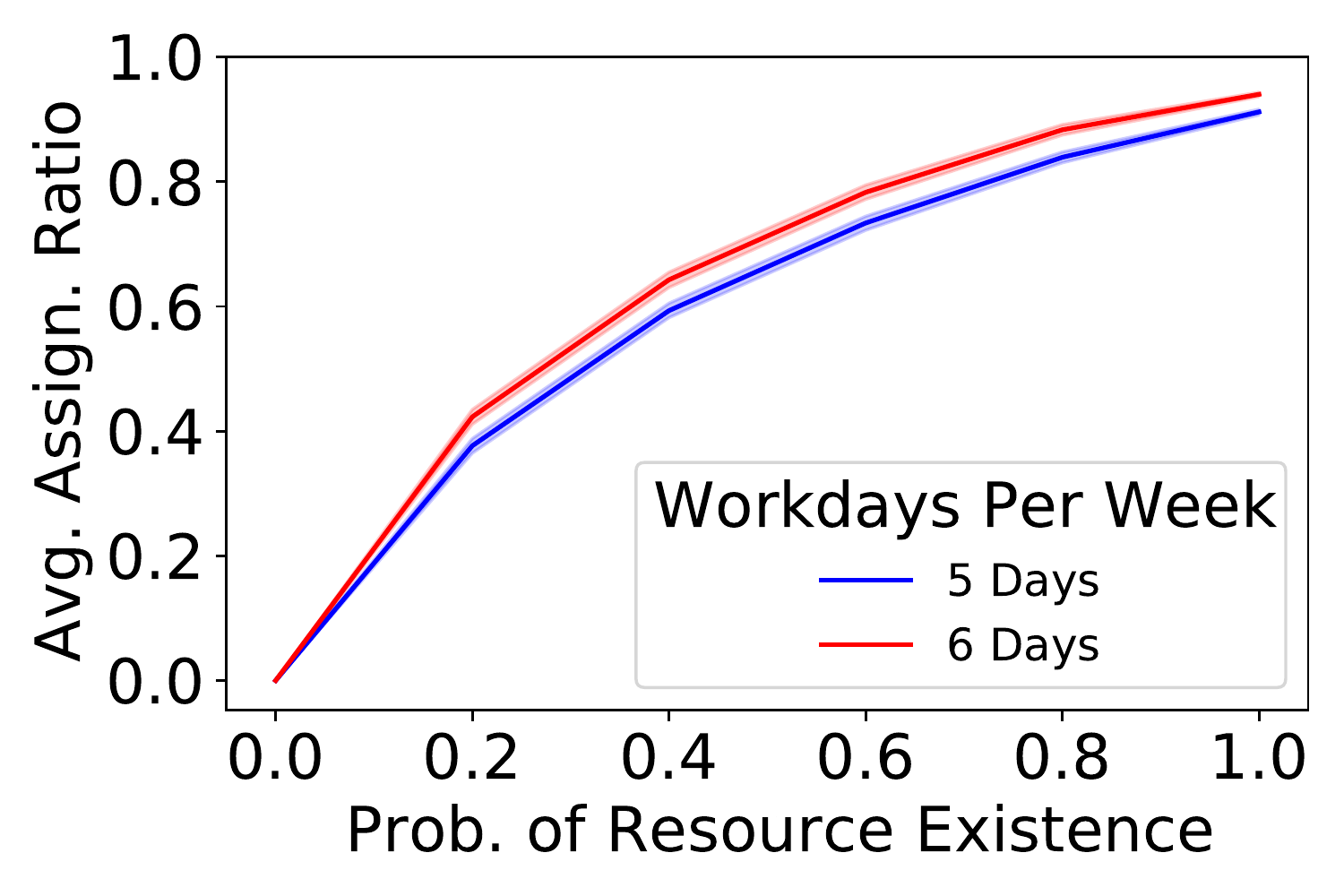}
\caption{\cc{} results. Average assignment ratio when
the number of resources is restricted.
\label{fig:prbresources}
}
\end{figure}

\medskip
\subsection{Additional Figures for Synthetic Graphs}
\medskip
To analyze the performance of the algorithms with respect to size of the
network, we consider complete bipartite graphs where the number of
agents~$n$ and resources~$m$ are the same. The number of agents (same as
resources) is varied from~50 to~200. There are~10 restrictions per agent
assigned costs randomly between~1 and~4. For each edge, we randomly assign
a restrictions subset of size at most five.  For each agent~$x_i$, the
number of rounds~$\rho_i$ is chosen randomly between~$1$ and~$5$.  
To generate~$K_i$,~$\rho_i$ rounds are sampled uniformly followed by choosing
the remaining rounds with probability~$0.5$.  
The results in
Figure~\ref{fig:synth} show the performance of \heuristicagmrm{} with
respect to \ilpagmaxsamrm{} in terms of solution quality and computation
time. We note that not only is the solution obtained using the
\heuristicagmrm{} close to the ILP-based algorithm, but it is also orders of magnitude faster.

\begin{figure}
\centering
\includegraphics[width=.48\columnwidth]{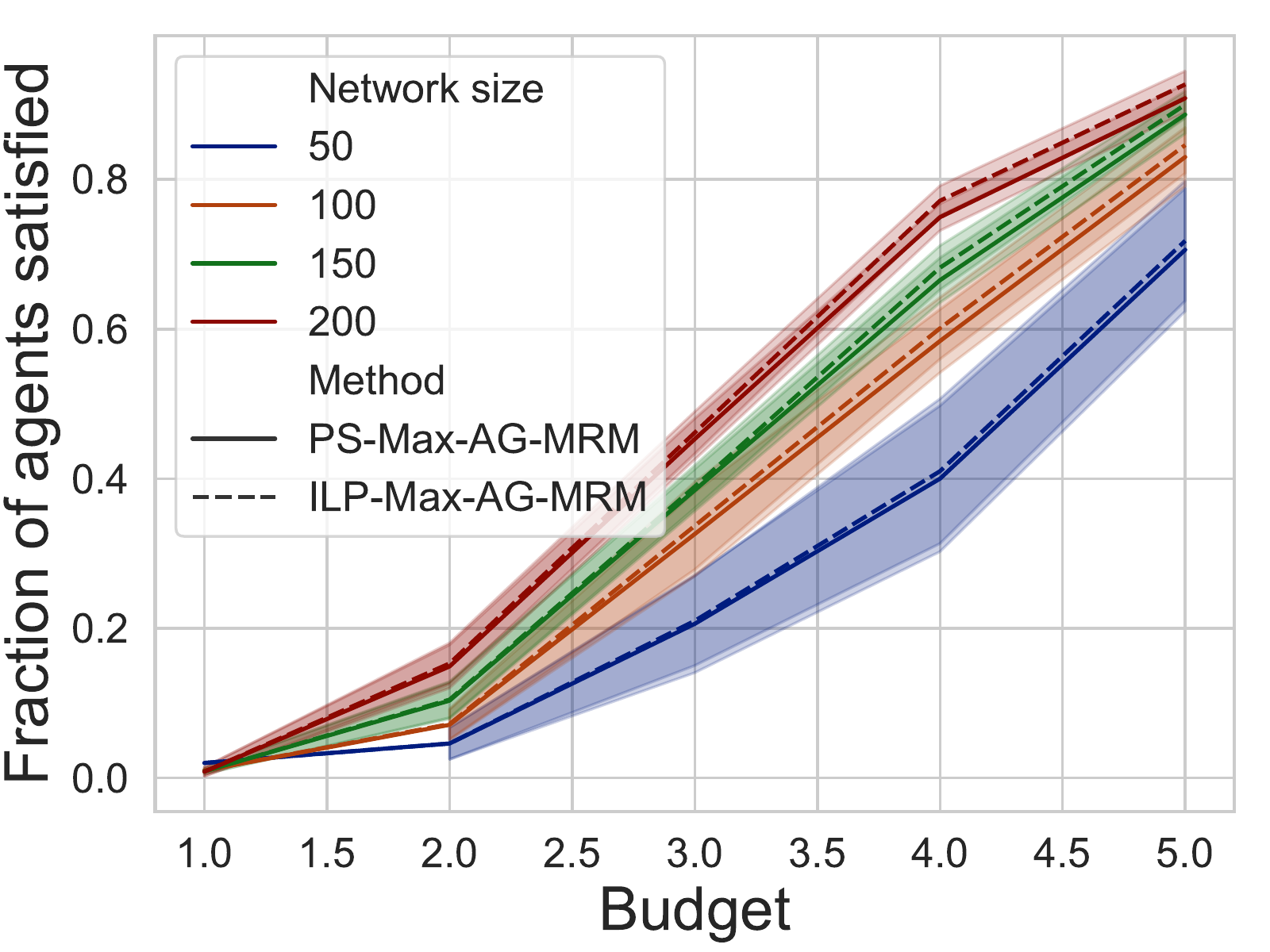}
\includegraphics[width=.48\columnwidth]{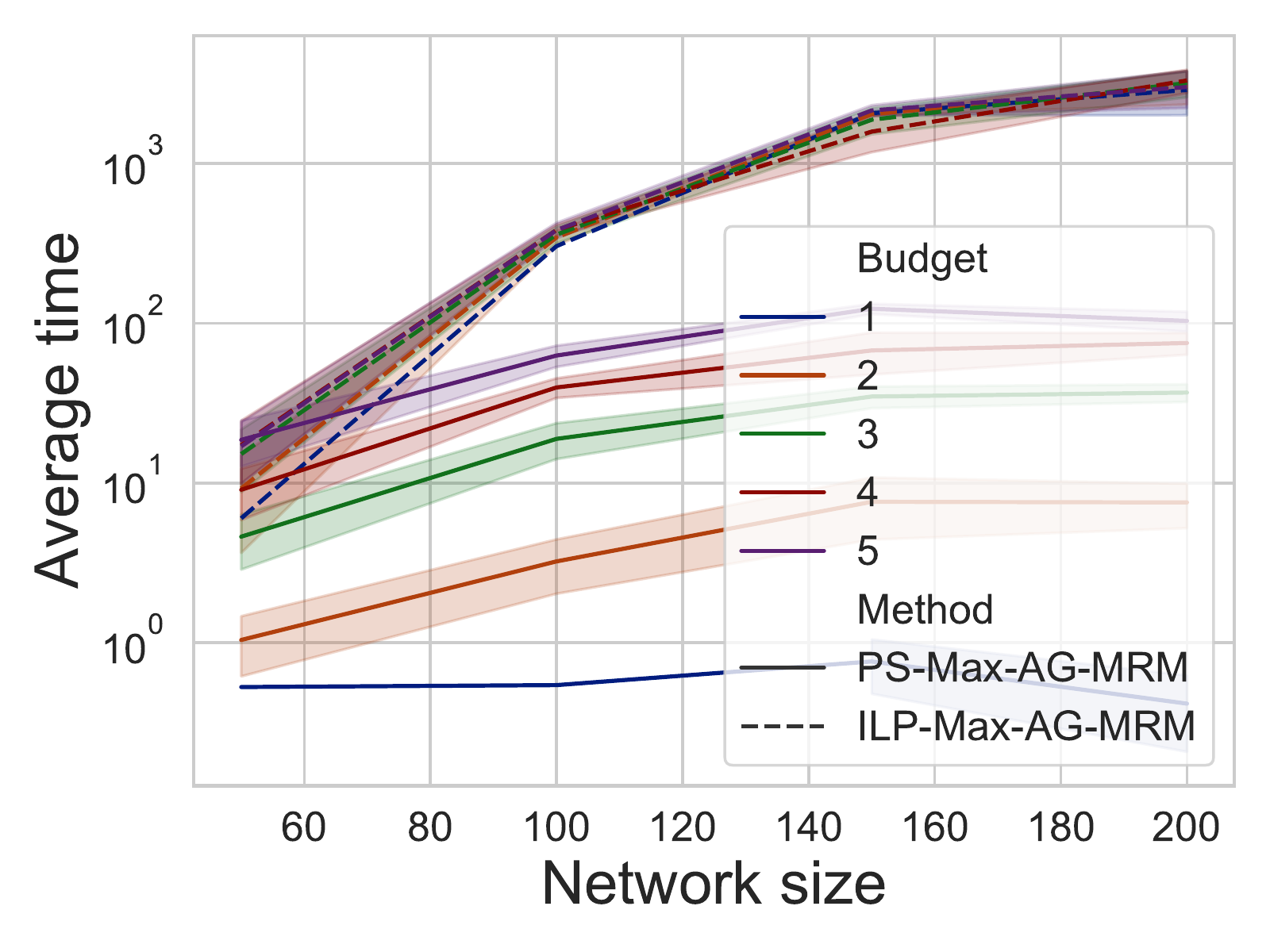}
\caption{Results for synthetic graphs: Solution quality and computation
    time.
\label{fig:synth}
}
\end{figure}

\medskip

\subsection{Implementation Details}

\medskip

All the software modules used to implement our algorithms and heuristics
were written in  Python 3.9. Graphs were generated and graph algorithms
were implemented using the Networkx package \cite{hagberg2008exploring}.
For local search, we used the simulated annealing implementation by Simple
AI \cite{simpleai}. We use “state” to be a set of relaxations per agent
(only the relevant ones as described).  Energy is the number of agents
satisfied or the number of successful matchings, depends on the specific
problem.  We use the geometric schedule\footnote{Kirkpatrick, Scott, C.
Daniel Gelatt, and Mario P. Vecchi. ``Optimization by simulated
annealing.'' Science 220.4598 (1983): 671-680.}. By trial and error, we
found the best setting of T1 to be 0.99 and T0 to be 100. If the
temperature decreased beyond 0.01, we set it to be 0.01. We stopped after
1000 iterations, moved back to the best solution if there was no
improvement in 40 iterations and returned the best found solution. We used
Gurobi Optimizer~\cite{gurobi} for solving the Integer Linear Problems.
The advice generation experiments were performed on a Linux machine equipped with two Intel
2.0 GHz Xeon Gold 6138 CPU processors (40 cores and 80 threads in total)
and 192 GB of memory. The other experiments, which required fewer resources, were performed using the 'Google Colab' platform.
In all places where random numbers were generated, we used the default seeds provided by python 3.9 and the Linux operating system we used.

\section{A General Result on Maximizing Utilitarian Welfare}
\label{sup:sec:gen_util_theorem}

\medskip

\begin{theorem}
\label{thm:total}
Suppose the benefit function $\mu_i$ of each agent $x_i$
is valid and strictly monotone 
increasing in the number of rounds. 
Any  solution that maximizes the total benefit 
corresponding to these benefit functions also maximizes
the utilitarian welfare, i.e., it maximizes the total number of rounds assigned to the agents.
\end{theorem}

\noindent
\textbf{Proof:}~
The proof is by contradiction. Suppose~$\mbsoln^*$ is a  solution
that provides the maximum benefit when each
benefit function satisfies the properties 
mentioned in the theorem. We will show that if there exists
a solution~$\mbsoln$ with number of total assignments greater than that
of~$\mbsoln^*$, then, that solution can be modified to a solution~$\mbsoln'$ with  benefit~$\totreward(\mbsoln')>\totreward(\mbsoln^*)$, contradicting the fact
that~$\mbsoln^*$ is an optimal solution.

For each agent $x_i$, Let $\gamma_i(\mbsoln)$ be the number of matchings
in which $x_i$ participates in $\mbsoln$.

\paragraph{Steps for creating $\mbsoln'$ given $\mbsoln$ and 
$\mbsoln^*$.}
\begin{itemize}
    \item Let~$t=0$ and initialize $\mbsoln'_t=\mbsoln$. 
    \item while$\big(\exists x_i$ for which $\gamma_i(\mbsoln^*)>
    \gamma_i(\mbsoln'_t) \big)$:
    \begin{itemize}
        \item Let~$\mbsoln'_{t+1}=\mbsoln'_t$
        \item pick $\gamma_i(\mbsoln^*)-\gamma_i(\mbsoln_t')$ rounds in 
        which $x_i$ participates in $\mbsoln^*$ but not in $\mbsoln_t'$. 
        \item In each round, match $x_i$ to the 
        same resource it was matched to in $\mbsoln^*$. If that resource is
        already matched to another agent in $\mbsoln_t'$, then replace that
        corresponding edge with the new edge.
        \item $t\leftarrow t+1$
    \end{itemize}
    \item return $\mbsoln'=\mbsoln'_t$.
\end{itemize}

\begin{lemma}
In each iteration of the algorithm, the total number of rounds assigned to all agents does
not decrease, i.e., the total number of rounds in~$\mbsoln'_t$ is at least as large  as that in~$\mbsoln'_{t-1}$.
\end{lemma}
\begin{proof}
In the construction of~$\mbsoln'_t$, for every edge added, at most one
edge is removed from~$\mbsoln'_{t-1}$.
\end{proof}

\begin{lemma}
The algorithm for creating $\mbsoln'$ terminates. 
\end{lemma}

\begin{proof}
First, in each iteration, the algorithm adds at least one edge from $\mbsoln^*$ to $\mbsoln'$.
Second, once an edge is added it will never be removed since each edge that is added belong to one of the matchings in $\mbsoln^*$.
\end{proof}

\begin{lemma}
For all  benefit functions that fulfil the conditions of the theorem, 
$\totreward(\mbsoln')>\totreward(\mbsoln^*)$. 
\end{lemma}
\begin{proof}
As before, we use $\gamma_i(\mbsoln)$ to denote the number of
rounds assigned to $x_i$ in solution $\mbsoln$.
Note that for each $i$,  $\gamma_i(\mbsoln') > \gamma_i(\mbsoln^*)$. 
By the contradiction assumption, the number of assignments in $\mbsoln$ is grater than that of $\mbsoln^*$ and therefore initially in $\mbsoln'$.
During the construction, the number of assignments does not decrease and therefore there exists at least one $i$ for which 
$\gamma_i(\mbsoln') > \gamma_i(\mbsoln^*)$.
Since the benefit function, $\reward_i$ is increasing for all $i$, it follows that $\totreward(\mbsoln')>\totreward(\mbsoln^*)$. 
\end{proof}

The last lemma above shows that $\totreward(\mbsoln')>\totreward(\mbsoln^*)$.
This contradicts the optimality of $\mbsoln^*$ and
the theorem follows. \hfill$\Box$
}
{}

\end{document}